\pdfoutput=1
\documentclass{article}

\usepackage{wrapfig}
\usepackage{microtype}
\usepackage{graphicx}
\usepackage{subcaption}
\usepackage{booktabs}
\usepackage{mleftright}
\usepackage{xspace}
\usepackage[round-mode=places,round-precision=0]{siunitx}

\usepackage{stmaryrd}
\usepackage{tikz}
\usetikzlibrary{decorations.pathreplacing}

\PassOptionsToPackage{table,svgnames,dvipsnames}{xcolor}

\usepackage[accepted]{icml2026}

\usepackage{amsmath}
\usepackage{mathtools}

\usepackage{amssymb}
\usepackage{wasysym}
\usepackage{latexsym}

\usepackage{amsthm}
\usepackage{thmtools} %
\AddToHook{cmd/appendix/before}{\crefalias{section}{appendix}\crefalias{subsection}{appendix}} %

\usepackage{microtype}
\usepackage{url}
\usepackage{enumitem}
\usepackage{tcolorbox}

\usepackage{xcolor}
\definecolor{accent}{rgb}{0, 0, 0.5}

\usepackage{hyperref}
\hypersetup{colorlinks=true,citecolor=accent,linkcolor=accent,urlcolor=accent}
\usepackage[capitalize,noabbrev,nameinlink]{cleveref}
\usepackage{doi}

\usepackage{natbib}

\usepackage{relsize}

\newtheorem{theorem}{Theorem}[section]
\newtheorem{proposition}[theorem]{Proposition}

\newtheorem{lemma}[theorem]{Lemma}
\newtheorem{corollary}[theorem]{Corollary}

\theoremstyle{definition}

\newtheorem{definition}[theorem]{Definition}

\newtheorem{example}[theorem]{Example}

 \newcommand{\macro}[1]{{#1}}

\newcommand{\CRASP}{\macro{\ensuremath{\mathsf{C\textsf{\textnormal{-}}RASP}}\xspace}}

\newcommand{\CRASPpos}{\macro{\ensuremath{\mathsf{C\textsf{-}RASP}}_{\!+}\xspace}}
\newcommand{\TLP}{\macro{\TL[\Psop]}}

\makeatletter
\newcommand{\oset}[2]{%
  {\mathop{#2}\limits^{\vbox to -.5\ex@{\kern-\tw@\ex@
   \hbox{\scriptsize $#1$}\vss}}}}
\makeatother
\newcommand{\countl}{\ensuremath{\oset{\leftharpoonup}{\#}}}

\newcommand{\TL}{\ensuremath{\mathsf{TL}}}

\newcommand{\Psop}{\mathord{\rlap{\makebox[\widthof{$\Diamond$}]{\raisebox{.72pt}{\scalebox{1.25}{\scriptsize$-$}}}}\raisebox{-0.5pt}{$\Diamond$}}}
\newcommand{\Hsop}{\boxminus}
\newcommand{\then}{\rightarrow}

\newcommand{\cif}[3]{\ensuremath{#1\;\mathbf{?}\; #2\; \textbf{:} \;#3}}

\newcommand{\N}{\mathbb{N}}
\newcommand{\Z}{\mathbb{Z}}

\newcommand{\bnfto}{\mathrel{::=}}

\newcommand{\rtfr}{$(1,\log n)$-precision transformer}

\newcommand{\rtfrs}{$(1,\log n)$-precision  transformers}
\newcommand{\Rtfrs}{$(1,\log n)$-precision  Transformers}

\newcommand{\frtfr}{$(1,1)$-precision transformer}

\newcommand{\frtfrs}{$(1,1)$-precision transformers}
\newcommand{\Frtfrs}{$(1,1)$-precision transformers}
\newcommand{\tfr}{\ensuremath{\macro{T}}}

\newcommand{\mat}[1]{\mathbf{#1}}
\renewcommand{\vec}[1]{\mathbf{#1}}
\newcommand{\str}[1]{#1}

\newcommand{\numterms}{\macro{m}}
\newcommand{\termind}{\macro{\ell}}

\newcommand{\depth}{\macro{\textnormal{dp}}}

\newcommand{\lang}{\macro{\mathcal{L}}}
\newcommand{\Lang}{\macro{\mathcal{L}}}

\definecolor{lowcolor}{HTML}{4CAF50}  %
\definecolor{highcolor}{HTML}{B22222} %

\newlength{\cellheight}
\newlength{\cellwidth}
\newlength{\bsize}
\newcommand{\bit}[2]{\langle#1\rangle_{#2}}
\newcommand{\bigbit}[2]{\left\langle#1\right\rangle_{#2}}

\newcommand{\bos}{\macro{\texttt{<BOS>}}}

\newcommand{\bitind}{\macro{b}}
\newcommand{\numbits}{\macro{p}}
\newcommand{\fracbits}{\macro{s}}

\newcommand{\sympred}[1]{\macro{#1}}

\newenvironment{calign}
  {
    \setcounter{equation}{0}%
    \footnotesize
    \vspace{-1em} 
    \flalign
  }
  {
    \endflalign
    \vspace{-2em}
  }

\newcommand{\diopheq}{\mathcal{D}}

\newcommand{\F}{\mathbb{F}}

\newcommand{\OMIT}[1]{}

\newcommand{\round}{\textnormal{round}}

\newcommand{\hypclass}{\mathcal{H}}
\newcommand{\only}[1]{\textrm{Only}_{#1}}

\newcommand{\prog}{\macro{P}}

\newcommand{\ialphabet}{\Sigma}

\newcommand{\complexity}{\macro{c}}

\newcommand{\pprec}[1]{\macro{\text{prec}\left(#1\right)}}
\newcommand{\pgirth}[1]{\macro{\text{girth}\left(#1\right)}}

\newcommand{\numlayers}{\macro{L}}

\newcommand{\dyck}{\macro{\textnormal{D}}}

\newcommand{\restricted}{\macro{\text{regulated}}}

\newcommand{\poly}{\textnormal{poly}}

\newcommand{\xbar}{\ensuremath{\underline{x}}}
\newcommand{\ybar}{\ensuremath{\underline{y}}}
\newcommand{\zbar}{\ensuremath{\underline{z}}}

\newcommand{\AY}[1]{\textcolor{green!80!orange}{(AY : #1)}}
\newcommand{\DC}[1]{\textcolor{blue}{(DC : #1)}}

\icmltitlerunning{Length Generalization Bounds for Transformers}

\begin{document}

\twocolumn[
  \icmltitle{Length Generalization Bounds for Transformers}

  \begin{icmlauthorlist}
    \icmlauthor{Andy Yang}{nd}
    \icmlauthor{Pascal Bergsträßer}{rptu}
    \icmlauthor{Georg Zetzsche}{mpi}
    \icmlauthor{David Chiang}{nd}
    \icmlauthor{Anthony W. Lin}{rptu,mpi}
  \end{icmlauthorlist}

  \icmlaffiliation{nd}{University of Notre Dame, USA}
  \icmlaffiliation{rptu}{University of Kaiserslautern-Landau, Germany}
  \icmlaffiliation{mpi}{Max-Planck Institute for Software Systems, Germany}

  \icmlcorrespondingauthor{Andy Yang}{ayang4@nd.edu}
  \icmlcorrespondingauthor{Pascal Bergsträßer}{bergstraesser@cs.uni-kl.de}
  \icmlcorrespondingauthor{David Chiang}{dchiang@nd.edu}
  \icmlcorrespondingauthor{Anthony Lin}{awlin@mpi-sws.org}

  \icmlkeywords{Length generalization, transformers, RASP}

  \vskip 0.3in
]

\printAffiliationsAndNotice{}  %

\begin{abstract}
  Length generalization is a key property of a learning algorithm that enables it to make correct predictions on inputs of any length, given finite training data. 
  To provide such a guarantee, one needs to be able to compute a length generalization bound, beyond which the model is guaranteed to generalize. 
  This paper concerns the open problem of the computability of such generalization bounds for \CRASP{}, a class of languages which is closely linked to transformers. 
  A positive partial result was recently shown by \citeauthor{chen2025nonasymptotic}~for
  \CRASP{} with only one layer and, under some restrictions, also with two layers.  
  We provide complete answers to the above open problem. 
  Our main result is the non-existence of computable length generalization bounds for \CRASP{} (already with two layers) and hence for transformers.
  To complement this, we provide a computable bound for %
  the positive fragment of \CRASP{}, which we show equivalent to fixed-precision transformers. 
  For both positive \CRASP{} and fixed-precision transformers, we show that the length complexity is exponential, and prove optimality of the bounds.
\end{abstract}

\section{Introduction}
The past few years have witnessed intensive research efforts to provide theoretically sound analyses of transformers \citep[e.g.][]{transformers-survey}. 
While initial efforts focused primarily on transformer expressivity, more recent efforts \cite{varre2025learning,pmlr-v291-wang25a,chen2025nonasymptotic,huang2024formalframeworkunderstandinglength} have delineated 
what is
efficiently learnable by transformers, particularly in terms of \emph{length generalization}. 
Intuitively, length generalization is when a model can be trained to accurately process strings of \emph{any length}, given only a finite training sample of strings with \emph{bounded length}.
Understanding when transformers can length-generalize has important practical implications. 
On the one hand, computational and data constraints limit the sequence lengths that are seen during training; on the other hand, processing large contexts and reasoning over long chains-of-thought create a need for strong capabilities even on long inputs.

Empirical studies have found that length generalization in transformers can vary from problem to problem \citep{zhou2024what} through mechanisms that remain unclear. 
Theoretical work has begun to provide some characterizations of the formal languages which transformers length-generalize on.
These characterizations are largely based on the expressivity of $\mathsf{RASP}$, a programming language which was designed to capture the expressive power of transformers \citep{pmlr-v139-weiss21a}.
\Citet{zhou2024what} conjectured that transformers would length-generalize on a language if and only if it had a short solution in the $\mathsf{RASP}$ variant $\mathsf{RASP\text{-}L}$, and provided empirical evidence to support this claim. 
Later, \citet{huang2024formalframeworkunderstandinglength} formalized this conjecture by proving that transformers are guaranteed to length-generalize under an idealized inference procedure for exactly those languages which are definable in $\CRASP$, a $\mathsf{RASP}$ variant expressively equivalent to transformers with fixed precision outside of attention \citep{yang2025kneedeep}. 
However, as pointed out by \citet{chen2025nonasymptotic}, the results of \citeauthor{huang2024formalframeworkunderstandinglength}~only guarantee learnability \emph{in the limit} \citep{Gold67}, that is, without providing any quantitative bounds on the training resources needed to learn a given language.  

For an example of how quantitative bounds lead to real-world impacts, we can look at scaling laws for language models. 
Language models are known to reliably follow scaling laws, which quantitatively describe how the test loss of a model relates to the size of the model and the training data \citep{kaplan2020scalinglawsneurallanguage}.
These laws can be used to derive the optimal model size and dataset size to minimize loss, given a fixed compute budget \citep{hoffmann2022an}. 
While these scaling laws may effectively predict loss, they fail to predict length generalization.
A growing body of evidence shows that length generalization is often wholly independent of conventional scaling laws; in many cases, increasing the model size or number of training examples are not the critical factors which determine length generalization \citep{anil2022exploring,nye2022show}.
As an example, \citet{nogueira2021investigatinglimitationstransformerssimple} find that models ranging from 50M to 3B parameters could not add 20-digit numbers when trained on addition of 15-digit numbers. 
At the same time, length generalization could be observed after training on 30 digits and testing on 60 digits, but increasing the amount of data past a critical threshold did not make much of a difference. 
A distinct theoretical framework is thus needed to provide quantitative guarantees for length generalization in transformers.

To provide a framework for these formal guarantees, \citet{chen2025nonasymptotic}
formulated the notion of \emph{non-asymptotic length generalization}, which asks
for a \emph{computable} bound $N$ 
such that a learning algorithm only requires labeled training data consisting of 
strings of length up to $N$ in order to correctly classify any test data containing strings of length greater than $N$. 
Such a learner is guaranteed to terminate, since $N$ is computable and there are finitely many strings of length up to $N$.
This notion of non-asymptotic length generalization is connected to the classical framework of \emph{exact learning} \citep{Angluin87} with only membership queries, as well as learning minimal hypotheses, which was already studied by Gold for finite automata \citep{Gold78}. 
In fact, \citeauthor{chen2025nonasymptotic}~show that non-asymptotic length generalization is equivalent to finite identification \citep[p.~457]{Gold67} when the latter is allowed to take into account the complexity of the ground-truth hypothesis.

We adopt this general learning framework in order to provide length generalization guarantees while abstracting away from the training dynamics of specific gradient-based algorithms, which are notoriously difficult to analyze. 
This framework assumes the best-case scenario of having access to all possible examples up to a specified length and optimizing perfectly on those examples. 
Thus, if length generalization is proven to be difficult under these favorable conditions, then it should be even harder for SGD on limited data (as in practice). 
In short, we have proven a bound that would still apply even if learning dynamics were taken into account.

To study non-asymptotic length generalization in transformers, \citet{chen2025nonasymptotic} proved computable length generalization bounds for $\CRASP$ with one layer and $\CRASP$ with two layers and additional
restrictions (see \cref{app:nalg_crasp_definitions}). 
These bounds are
(respectively) polynomial and exponential in the maximum absolute value of
constants appearing in the $\CRASP$ program. 
For example, using this, one can surmise that learning the $\mathsf{MAJORITY}$ language $\{ w \in \{a,b\}^* : \text{$w$ has more $a$'s than $b$'s} \}$
using $\CRASP$ programs of depth~$1$ requires only strings up to length quadratic in the absolute value of constants. 
Similarly, \citet{izzo2025quantitative} %
established bounds on the length at which length generalization occurs in one- and two-layer transformers with some additional modifications (see \cref{sec:izzo}). %
These results left open whether general $\CRASP$ programs and transformers with more than two layers also admit non-asymptotic length generalization.

\paragraph{Contributions.} The main contribution of this paper is to answer the open problem of the non-asymptotic length generalization of transformers and \CRASP{} in the negative. 

\begin{theorem}[Informal version of \cref{cor:crasp_undecidable_complexity}]
    There is no terminating algorithm for perfectly learning a $\CRASP$ program $P$ (given an upper bound on the size of $P$), even if $P$ only has depth two. Thus, no such algorithm exists for transformers of depth two or beyond. 
\end{theorem}
The consequence for transformers is a corollary of the depth-preserving equivalence between \CRASP{} and transformers 
\citep{yang2025kneedeep}.
In particular, the lengths of strings necessary for length generalization must
grow faster than any computable function.

\OMIT{
computable, that is, larger than $\Omega(\exp(\text{BB}(|P|)))$ (exponential in the busy beaver function \citep{rado1962on}).\footnote{Proof sketch: \Citet{jones1984register} encode register machines as Diophantine equations, where the size of the solution scales exponentially with the runtime of the (halting) machine.
From this, we can derive that the length complexity of any class of $\CRASP$ programs must grow at least exponentially in the runtime of halting Turing machines of that size.
}

}

To complement this uncomputability result, our secondary contribution is to provide a tight computable (in fact, \emph{exponential}) length generalization bound for the positive fragment $\CRASPpos$ of $\CRASP$, which is expressively equivalent to fixed-precision transformers. This is a natural restriction \citep{li2025characterizing} since real-world transformers are implemented on a device that supports only a fixed finite precision (i.e. floating-point computation).

\begin{theorem}[Informal version of \cref{cor:exp-len-comp}]
To perfectly learn a $\CRASPpos$ program $P$ (given an upper bound on the size of $P$), it is sufficient to train on strings with length exponential in the size of $P$, and, in the worst case, it is necessary to see at least one string of exponential length. The same bound holds for a %
transformer $\tfr$.
\end{theorem}

\paragraph{Overview.}
To show the non-existence of computable bounds for length generalization in $\CRASP$ and transformers, we use the fact that non-asymptotic length generalization is equivalent to the decidability of the language equivalence problem for any finite class of languages \citep[Lemma~3.4, Theorem~3.2]{chen2025nonasymptotic}.
We note that, for $\CRASP$-definable languages,  decidability of language equivalence is equivalent to decidability of nonemptiness, and 
we prove undecidability of emptiness of $\CRASP$-definable languages via a reduction from the undecidability of Hilbert's 10th problem \citep{hilbert2019mathematical}.

To show exponential bounds for length generalization in $\CRASPpos$ and the equivalent class of \frtfrs{}, we show how $\CRASPpos$ is equivalent to the unary temporal logic $\TLP$ with only the (strict) past operator,
and therefore defines the same class of languages as \frtfrs{} \citep{li2025characterizing}. 
This translation into $\TLP$ incurs an exponential blow-up.
Then, since each satisfiable formula $\varphi$ in $\TLP$ has a witnessing string of length polynomial in $|\varphi|$, we obtain an exponential length generalization bound for $\CRASPpos$. 
We also show that this exponential bound is tight in the worst case.

We note that previous work has also investigated the complexity of emptiness checking for different transformer variants (though these results do not imply ours). 
\Citet{salzer2025transformer} showed the first undecidability result for log-precision transformer encoders with average-hard attention layers. 
Later, \citet{salzer2025counting} established undecidability of language emptiness (and therefore of equivalence) for uniform-attention transformers.
In general, $\CRASP$ is a more restricted class than all of the above; thus, our undecidability result is stronger than previous results. 
For the case of unique-hard attention transformers --- which are able to
express LTL (i.e. order-sensitive) properties \cite{BKLP24,YCA24} ---
\citet{bergstrasser2025transformers} showed that emptiness checking is $\mathsf{EXPSPACE}$-complete.

\OMIT{
When a model generalizes from to short inputs seen at training time to long inputs at inference time, we say it exhibits length generalization.
In the case of transformer-based language models, studies have found that the length generalization capabilities of transformers may vary significantly from task to task \citep{anil2022exploring,kazemnejad2023the}.
Understanding when length generalization occurs is important due to the demands placed upon language models; computational and data constraints limit the sequence lengths seen during training, while processing large contexts and reasoning over long chains-of-thought create a need for performance on long inputs. 

Theoretical work provided characterizations of which tasks transformers can and cannot length generalize on.
\citet{zhou2024what} conjectured that transformers would length generalize on a task if and only if it had a short solution in $\mathsf{RASP}$-l, a programming language based on $\mathsf{RASP}$ \citep{pmlr-v139-weiss21a}.
Later, \citet{huang2024formalframeworkunderstandinglength} formalized this $\mathsf{RASP}$-l conjecture using $\CRASP$, showing that transformers would length generalize on tasks which had a $\CRASP$ program, under an idealized inference procedure. 
}

\section{Preliminaries}
\label{sec:prelim}

\subsection{Notation}
We write $\mathbb{N}$ for the 
set of natural numbers including $0$. We write $[n]$ for the set $\{1, \ldots, n\}$.
Let $\ialphabet$ be a finite alphabet. The set of all strings over $\ialphabet$
is denoted by $\ialphabet^*$. If $\str{w} \in \ialphabet^*$ and $\sigma \in \ialphabet$, we write 
$|\str{w}|_\sigma$ for the number of occurrences of $\sigma$ in $\str{w}$. We
write $|\str{w}|$ for the length of $\str{w}$ and $\str{w}_i$ for the symbol at position $i \in [|\str{w}|]$. By $\ialphabet^{\leq n}$ with $n
\in \N$, we mean the set of all strings $w \in \ialphabet^*$ of length at most
$n$.

\subsection{$\CRASP$ (Counting $\mathsf{RASP}$)}
We define the syntax and semantics of $\CRASP$ (akin to the logics of \citet{yang2024counting,BKLP24}). In the sequel, we will define
other logics $\CRASPpos$ and $\TLP$ as fragments of $\CRASP$. 
\begin{definition}
\label{def:TLC_semantics}
The syntax of $\CRASP$ formulas is defined:
    \begin{align*}
        \phi & \bnfto \sympred\sigma \mid \Psop \phi \mid \Hsop\phi \mid \lnot \phi_1 \mid \phi_1\land \phi_2 \mid \sum _{t\in \mathcal{T}} \alpha_t t \sim k \\
        t & \bnfto \countl[\phi_1] %
        \mid c &&& 
    \end{align*} 
    where $\sympred\sigma\in\Sigma$, $\alpha_i,k,c\in\Z$ and $\mathord\sim\in\{<,\leq, =,\geq,>\}$.
    The semantics of formulas is defined as follows:
    \begin{subequations}
        \let\origiff\iff
        \renewcommand{\iff}{\hspace{\tabcolsep}\origiff\hspace{\tabcolsep}}
        \begin{alignat*}{2}
            &\str{w},i \models \sympred\sigma &\iff &\text{$\str{w}_i = \sigma$} \\
            &\str{w},i \models \lnot \phi & \iff &\str{w},i\not\models \phi\\
            &\str{w},i \models \phi_1\land \phi_2 &\iff &\text{$\str{w},i \models \phi_1$ and $\str{w},i \models \phi_2$} \\
            &\str{w}, i \models \Psop\phi &\iff& \text{$\str{w}, j \models \phi$ for some $j<i$}\\
            &\str{w}, i \models \Hsop\phi &\iff& \text{$\str{w}, j \models \phi$ for all $j\leq i$}\\
            &\str{w},i \models \sum _{t\in \mathcal{T}} \alpha_t t \sim k &\iff &\sum_{t\in \mathcal{T}} \alpha_t t^{\str{w},i} \sim k.
        \end{alignat*}
    \end{subequations}
    The semantics of terms is defined as follows:
    \begin{subequations}
        \begin{align*}
            \countl[\phi]^{\str{w},i} &= |\{j\in[1,i] \mid \str{w},j \models \phi\}| \\
            c^{\str{w},i} &= c.
        \end{align*}
    \end{subequations}
    We write $\str{w} \models \phi$ if $\str{w}, |\str{w}| \models \phi$, and we say that $\phi$ defines the language $\lang(\phi) = \{\str{w} \mid \str{w} \models \phi\}$.
\end{definition}
First, $\Psop \phi$ (``previously'') can be viewed as syntactic sugar for $\countl[\phi]\geq 2 \lor (\lnot \phi \land \countl[\phi]\geq 1)$; it is true at position $i$ iff $\phi$ is true for some position $j< i$.
Secondly, $\Hsop \phi$ (``historically'') can be viewed as syntactic sugar for $\countl[\neg\phi] = 0$; it is true at position $i$ iff $\phi$ is true at all positions $j\le i$.
Finally, $\Hsop \phi$ is also equivalent to $\phi\land \lnot(\Psop\lnot\phi)$, noting the strictness of $\Psop$ and the non-strictness of $\Hsop$. 

In the sequel, we will use a \emph{DAG} (directed acyclic graph) representation of $\CRASP$ formulas,
where a subformula $\varphi$ may be used \emph{multiple times} in a formula. 
Such a formula can be thought of as a straight-line \emph{program}, i.e.,
a sequence $\varphi = (\varphi_i)_{i=1}^n$, where $\varphi_i$ is any $\CRASP$
definition that could refer to $\varphi_j$ with $j < i$. The \emph{size}
$|\varphi_i|$ of a definition is the number of symbols, 
where we assume constants to be encoded in binary and each reference to $\varphi_j$ for $j < i$ to be of size 1.
Then the size $|\varphi|$ of 
$\varphi$ is defined to be $\sum_{i=1}^n |\varphi_i|$. 
For
example, the formula $\bigwedge_{a \in \ialphabet} (\varphi \to \countl a \ge k)$ --- 
which says that if $\varphi$ is true, then every $a$ occurs in the (non-strict) past at least $k$ times
--- can be represented by a program of size $O(|\varphi| + |\ialphabet|\log(k))$.
Note that $|\varphi|$ is counted once (not $|\ialphabet|$ times).

\OMIT{
\DC{in the B-RASP paper we called them operations. here, we have called them declarations, lines, and steps. I suggest sticking with the original}
\AY{Will do}
\begin{definition}
    A $\CRASP$ \emph{declaration} is an expression of the form $P:=\phi$ where $\phi$ is a $\CRASP$ formula and $P$ is a program line name. 
    A $\CRASP$ \emph{program} $\prog$ is a sequence of $\CRASP$ declarations. \DC{confusing to use the same letter twice}
    The output of the program is derived by evaluating the semantics of each program line as $\CRASP$ formulas, line-by-line, and taking the output of the last line at the final position. 
\end{definition}
\AY{Would it be helpful to put a very formal definition in the appendix?}

We define measures of complexity for $\CRASP$ programs:
}

\begin{definition}
\label{def:TLC_depth}
The \emph{depth} of formulas and terms is defined:
\begin{align*}
\depth(\sympred\sigma)=\depth(c)  &= 0 \\
\depth(\neg \phi) &= \depth(\phi) \\
\depth(\phi_1 \land \phi_2) &= \max \{ \depth(\phi_1), \depth(\phi_2) \} \\
\depth(\countl[\phi])=\depth(\Psop \phi) =\depth(\Hsop \phi) &= \depth(\phi)+1\\
\depth\left(\sum _{t\in \mathcal{T}} \alpha_t t \sim k\right)&=\max_{t\in\mathcal{T}} \depth(t).
\end{align*}
We write $\CRASP_k$ for the set of all $\CRASP$ formulas with depth at most $k$. 
\end{definition}
The \emph{depth} of a $\CRASP$ program $P$ is $\max_{\phi\in P} \depth(\phi)$.
The \emph{precision} $\pprec{P}$ is the number of bits needed to encode its largest constant.
The \emph{girth} $\pgirth{P}$ is the maximum number of summands occurring in any sum of $P$.
\begin{example}\label{example:dyck}
First, define $\dyck_k$, the Dyck language of bounded depth $k$, as the language of strings with the following regular expression defined inductively:
    \begin{align*}
        \dyck_1&=(ab)^*\\
        \dyck_{k+1}&=(a\dyck_{k}b)^*
    \end{align*}
This is the language of balanced parentheses, using $a$ and $b$ instead of $($ and $)$, with a maximum nesting depth of $4$. Below is a $\CRASP$ program which recognizes $\dyck_4$:
\begin{tcolorbox}[title={Dyck of depth 4} ]
    \begin{calign}
        \phi_{\text{lower}} &:=  \countl a -\countl b \geq 0 & \\
        \phi_{\text{upper}} &:=  \countl a -\countl b \leq 4 & \\
        \phi_{\text{bounded}} & :=  \phi_{\text{lower}} \land \phi_{\text{upper}}& \\
        \phi_{\text{all\_bounded}} & := \Hsop \phi_{\text{bounded}} & \\
        \phi_{\text{balanced}} & :=  \countl a = \countl b & \\
        \phi_{\dyck_4} &:= \phi_{\text{all\_bounded}} \land \phi_{\text{balanced}} &  
    \end{calign}
\end{tcolorbox}
This program has depth $2$, girth $2$, and precision $2$.  
\end{example}

\subsection{Transformers}

Our results concern two variants of transformers in this paper, which round to finite-precision in different ways. 

\begin{definition}
    A $(p,q)$-precision transformer is one which uses $O(p)$ bits of precision outside of attention and $O(q)$ bits of precision inside of attention.
\end{definition}
Please see \cref{sec:transformer_definitions} for precise definitions. 

\subsection{Computational Learning Theory}
We discuss notions from computational learning theory (cf. \citealp[Chapter
1.2.2]{COLT-book}) instantiated to formal languages. 
Our learning algorithm learns a \emph{hypothesis} (a language $S 
\subseteq \ialphabet^*$).
Let $\hypclass$ be a set of possible hypotheses that can be learned. 
A \emph{representation scheme} $\Lang$ for $\hypclass$ is a
surjective partial function from strings in $\Gamma^*$ to hypotheses in 
$\hypclass$. If $\Lang(E) = w$, we say that $E$ \emph{represents} $w$.

This allows us to measure the \emph{size} of a hypothesis by the length of its shortest representation. The \emph{descriptional complexity}
of a hypothesis $L \in \hypclass$ with respect to $\Lang$ is the length of the shortest 
representation(s) for $L$, that is, $\min \{ |E| : E \in \Gamma^*, \Lang(E) = L \}$.

\OMIT{
In our setting, $\hypclass$ will be taken as the set of $\CRASP$-definable languages, $\Gamma^*$ to be the set of all $\CRASP$ formulas, $|E|$ for $E\in\Gamma^*$ to be the size of a formula, and $\Lang(E)$ defined as the language recognized by a formula.
}

\subsection{Non-Asymptotic Length Generalization}

Suppose we want to learn a hypothesis $L$, and we know that $L$ has descriptional complexity (with respect to $\Lang$) at most $n$. \emph{Up to what string length} do we need to see training strings, so that we can learn a representation $E$ with $\Lang(E) = L$ of size at most $n$?

The notion of \emph{length complexity} gives a way to answer this question.
\begin{definition}[Length complexity] 
    Given a hypothesis class $\hypclass$ and a representation scheme $\Lang$ for
    $\hypclass$, the \emph{length complexity} of $\hypclass$ with respect to
    $\Lang$ is the minimal function $f_\Lang \colon \N \to \N$ such that 
    for any two hypotheses
    with descriptional complexity (with respect to $\Lang$) at most $\complexity$, there is a string of length at most $f_\Lang(\complexity)$ that distinguishes them.
    That~is, 
    \[
            f_\Lang(\complexity) = \max_{\substack{E,E' \in \Gamma^{\le \complexity} \\ \Lang(E) \setminus \Lang(E') \ne \emptyset}} \min\{ |w| : w \in \Lang(E) \setminus \Lang(E') \}.
    \]
\end{definition}
If $f_\Lang(\complexity)$ has a \emph{computable} upper bound on the maximum string length, then it is possible in principle to learn any language $L \in \hypclass$ perfectly, in the following way:
\begin{enumerate}
\item Receive the maximum descriptional complexity $\complexity \in \N$.
\item\label{item:bound} Compute a maximum string length $N \in \N$.
\item\label{item:membership} Receive training data $T = \{\str{w} \in L : |\str{w}| \le N \}$.
\item\label{item:learner} Output representation $E \in \Gamma^*$ such that $|E|
    \le n$, $\Lang(E) \cap \ialphabet^{\leq n} = T$, and $\Lang(E) = L$.
\end{enumerate}
Step \ref{item:membership} is computable because of our assumption that membership is
decidable for $\Lang$.
The learner (step \ref{item:learner}) works, in principle, by enumerating all
possible hypotheses (as there are only finitely many of them with descriptional
complexity at most $n$) and checking each one against the training data.
Assuming that the true hypothesis 
has descriptional complexity (with respect to
$\Lang$) at most $\complexity$, length generalization ensures the uniqueness of 
$\Lang(E)$.
This definition is akin to the problem of finding a minimum representation in computational learning theory \citep[e.g.,][]{COLT-book}.

Therefore, whether the length complexity of $\hypclass$ (with respect to $\Lang$) can be computably bounded (step \ref{item:bound}) is an indication of whether learning length generalization is possible.
This was shown by \citet{chen2025nonasymptotic} to be equivalent to decidability of language equivalence.
\begin{proposition}[{\citealt[Lemma~3.4]{chen2025nonasymptotic}}]
    For any hypothesis class $\hypclass$ over formal languages and
    representation scheme $\Lang: \Gamma^* \to \hypclass$ with decidable
    membership problem, there is a computable upper bound on length complexity 
    for $\hypclass$ with respect to $\Lang$ iff language equivalence (that is, given $E,E' \in \Gamma^*$,
    check whether $\Lang(E) = \Lang(E')$) is decidable.
\end{proposition}

In our setting, $\hypclass$ will always be a set of formal languages $L \subseteq
\ialphabet^*$. 
In particular, we take $\hypclass$ to be various subset of
formal languages expressible by transformers, i.e., $\CRASP$-definable languages.
We assume a representation scheme $\Lang$ such that checking $w \in
\Lang(E)$, for any given string $w \in \ialphabet^*$, is decidable.
In particular, this is the case if we use $\CRASP$ or any equivalent representation of $\CRASP$, for example, \rtfrs{}.

\section{Undecidability of $\CRASP$}\label{sec:undecidable}

Our main result establishes the undecidability of language emptiness for $\CRASP$, thus answering an open question of \citet{chen2025nonasymptotic}.

\begin{theorem} \label{thm:undecidable}
It is undecidable whether a given \CRASP{} program defines the empty language~$\emptyset$ or not.
\end{theorem}

We prove this by reduction from solvability of Diophantine equations. The decidability of Diophantine equations was Hilbert's Tenth Problem and answered negatively by \citet{matiyasevich1993hilbert}.

Any Diophantine equation can equivalently be expressed as a system of equations of the form $x=c$, $x+y=z$, or $x\cdot y = z$, where $x,y,z \in \mathcal{V}$ are variables ranging over $\mathbb{N}$ and $c \in \mathbb{N}$ is a constant \citep[p.~3]{matiyasevich1993hilbert}. Furthermore, we may assume without loss of generality that in any equation, all the variables are distinct. (If an equation contains $x$ twice, rename one of them to $x'$ and add the equations $x' = x + e$, $e = 0$.) 

Given a system of equations over $\mathcal{V}$, let $\Sigma = \mathcal{V} \cup \{\xbar : x \in \mathcal{V}\}$.
We will construct a $\CRASP$-definable language $\lang \subseteq \Sigma^*$ such that $\lang \neq \emptyset$ iff the equations have a solution.

\begin{definition}
We say that a language $\lang \subseteq \Sigma^*$ \emph{encodes} an equation $\diopheq$ over $\mathcal{V} \subseteq \Sigma$ if 
\begin{itemize}
\item If $\str{w} \in \lang$, then $x_1 \mapsto |\str{w}|_{x_1}, \ldots, x_m \mapsto |\str{w}|_{x_m}$ satisfies~$\diopheq$.
\item For any $n_1, \ldots, n_m \in \mathbb{N}$ satisfying $\diopheq$, there is a string $\str{w} \in \lang$ such that $|\str{w}|_{x_i} = n_i$ for all $i \in [m]$. 
\end{itemize}
\end{definition}

\begin{definition}
A formula of $\CRASP$ is \emph{\restricted{}} if no atomic formula $\sympred{\sigma}$ (for any $\sigma \in \Sigma$) occurs outside of a counting operator $\#$. Let \emph{\restricted{} $\CRASP$} be the set of all \restricted{} formulas of $\CRASP$.
\end{definition}

First, we will show that equations of the aforementioned forms can be encoded by languages definable in \restricted{} $\CRASP$.
Then we will show how to combine these languages to obtain the result.

\subsection{Encoding One Equation}
\begin{lemma} \label{lem:single_equation}
Any equation $x=c$, $x+y=z$, or $x\cdot y = z$ can be encoded by a language definable in \restricted{} $\CRASP$.
\end{lemma}
\begin{proof}
The first two cases are easy; multiplication is more difficult.

For any $A \subseteq \Sigma$, define the following formula, which restricts strings to only use symbols in $A$: 
\[ \only{A} := \bigwedge_{\sigma \in \Sigma \setminus A} \countl \sympred{\sigma} = 0. \]
\paragraph{Constants} An equation $x = c$ is encoded by the language $\lang = \{ x^c \}$, which is defined by the formula $\left(\countl \sympred{x} = c \land \only{x}\right)$.
\paragraph{Addition} An equation $x + y = z$ is encoded by the language $\{ \str{w} \in \{x,y,z\}^* : |\str{w}|_x+|\str{w}|_y=|\str{w}|_z\}$, which is defined by the formula $\left(\countl \sympred{x} + \countl \sympred{y} = \countl \sympred{z} \land \only{x,y,z}\right)$.
\paragraph{Multiplication}
An equation $x\cdot y = z$ is encoded by the language $\lang = \{x^n(z^ny\zbar^n\ybar)^m : m,n \in \mathbb{N} \}$, which encodes $x \cdot y = z$. We need to show that $\lang$ is definable in \restricted{} $\CRASP$.

First, we require that $y$'s and $\ybar$'s strictly alternate, which is equivalent to:
(1) at each position, the number of $y$'s exceeds the number of $\ybar$'s by $0$ or $1$ and (2) at the end of the string, the number of $y$'s and $\ybar$'s are equal.
\begin{tcolorbox}[title={$y$ and $\ybar$ strictly alternating} ]
\begin{calign}
    \phi_{1} &:= \left(\countl \sympred{y} = \countl \sympred{\ybar} \lor \countl \sympred{y} = \countl \sympred{\ybar} + 1\right) & \\
    \phi_{2} &:= (\countl \sympred{y} = \countl \sympred{\ybar}) & \\
    \textrm{Alt}_y &:= \Hsop \phi_1 \land \phi_2 \land \only{y,\ybar}& 
\end{calign}
\end{tcolorbox}
We extend this formula to define $x^*(z^*y\zbar^*\ybar)^*$.
We need that (3) no other symbols precede an $x$, 
(4) $z$ only occurs when the numbers of $y$'s and $\ybar$'s are equal, and (5) $\zbar$ only occurs when the number of $y$'s exceeds the number of~$\ybar$'s.

\begin{tcolorbox}[title={$x^*(z^*y\zbar^*\ybar)^*$} ]
\begin{calign}
    \phi_{3} &:= \sympred{x} \then \left(\countl \sympred{y} = 0 \land \countl \sympred{z} = 0\right) & \\
    \phi_{4} &:= \sympred{z} \then \left(\countl \sympred{y} = \countl \sympred{\ybar}\right) & \\
    \phi_{5} &:= \sympred{\zbar} \then \left(\countl \sympred{y} = \countl \sympred{\ybar}+1\right) & \\
    \phi_{6} &:= \Hsop \phi_3 \land \Hsop \phi_4 \land \Hsop\phi_5 \\
    \textrm{Ord}_{x,y,z} &:= \textrm{Alt}_y \land \phi_{6} \land \only{x,y,z,\ybar,\zbar} &
\end{calign}
\end{tcolorbox}

Now to define $x^n(z^ny\zbar^n\ybar)^m$, we need the additional constraints that all blocks of $x$,$z$, and $\zbar$ have the same size.
This can be achieved by asserting that (6) at every $y$, the difference between $\countl \sympred{z}$ and $\countl \sympred{\zbar}$ is equal to $\countl \sympred{x}$, and (7) at every $\ybar$, they are balanced.

\begin{tcolorbox}[title={$x^n(z^ny\zbar^n\ybar)^m$} ]
\begin{calign}
    \phi_{7} &:= \sympred{y} \then \left(\countl \sympred{x} = \countl \sympred{z} - \countl \sympred{\zbar}\right) & \\
    \phi_{8} &:= \sympred{\ybar} \then \left(\countl \sympred{z} = \countl \sympred{\zbar}\right) & \\
    \phi_{x \cdot y = z} &:= \textrm{Ord}_{x,y,z} \land \Hsop\phi_7 \land \Hsop\phi_8 & 
\end{calign}
\end{tcolorbox}
\end{proof}

\subsection{Combining Equations}

First, we 
show how to concatenate multiple languages that encode equations.

\begin{lemma}\label{lem:crasp_rpol}
    Let $\Sigma$ be an alphabet and $\$\not\in \Sigma$. For any languages $\lang_1, \ldots, \lang_n$ over $\Sigma$ such that each $\lang_i$ for $i \in [n]$ is definable in \restricted{} $\CRASP$, the language $\lang_1\$\lang_2\$\cdots\lang_n$ is definable in $\CRASP$.
\end{lemma}
\begin{proof}
For any $i \in [n]$, let $\phi_i$ be a \restricted{} formula defining $\lang_i$. Let $\phi_i'$ be the result of taking $\phi_i$ and 
replacing every subformula $\countl \phi$ with $\countl((\countl \sympred{\$} = i-1) \land \phi)$.
Observe that $\phi_i' \land (\countl \sympred{\$} = n)$ defines the language $\{ \str{u} \$ \str{v} \$ \str{w} : |\str{u}|_\$ = i-2, \str{v} \in \lang_i, |\str{w}|_\$ = n-i \}$. 
Then construct the formula \begin{equation*} \phi_{\lang} := \phi_1' \land \cdots \land \phi_n' \land (\countl \sympred{\$} = n). \tag*{\qedhere} \end{equation*}
\end{proof}
This allows us to reduce the solving of Diophantine equations to emptiness-checking in $\CRASP$. 

\begin{proof}[Proof of \cref{thm:undecidable}]
    Consider any Diophantine equation $\diopheq$ with $\mathcal{V}$ as the set of all variables. 
    We can rewrite it as a system of equations $\diopheq_i$ of the form $x=c$, $x+y=z$, and $x\cdot y = z$, where $x,y,z$ are variables ranging over $\mathbb{N}$ and $c \in \mathbb{N}$ is a constant. 
    For each $\diopheq_i$, by \cref{lem:single_equation} there is a $\CRASP$-definable language $\lang_i$ that encodes $\diopheq_i$.
    By \cref{lem:crasp_rpol} there is a $\CRASP$ formula $\phi_{\lang}$ defining $\lang_1\$\lang_2\$\cdots\$\lang_n$.
    It remains to add additional constraints on $\phi$ to ensure that each $\lang_i$ has the same number of occurrences of each variable.

    We add for each variable $x$ the following conjunction:
    $$\psi_x:=\bigwedge_{i=1}^{n-1} \left(\countl[x\land \countl[\$]=i]=\countl[x\land \countl[\$]=i+1]\right).$$

    Then we define the following formula
    \[\phi:=\phi_{\lang}\land\bigvee_{x\in\mathcal{V}} \psi_x.\]

    Finally, we can see that $\lang(\phi)$ is nonempty iff $\diopheq$ has a solution. 
\end{proof}

\begin{corollary}\label{cor:crasp_undecidable_complexity}
The length complexity of $\CRASP$ is not computably bounded.
\end{corollary}
\begin{proof}
By Lemma~4.5 of \citet{chen2025nonasymptotic}, the length complexity is computably bounded iff language equivalence is decidable, which is decidable iff language emptiness is decidable. 
The result follows from \cref{thm:undecidable}.
\end{proof}

This applies even to $\CRASP$ programs of depth $2$, as the formulas discussed above are all of depth $2$.
We remark that in general, $\CRASP$ can be reduced to $\CRASP$ with two layers while preserving language nonemptiness.
This is shown in \cref{app:crasp_two}.

\subsection{Relationship to Previous Bounds}

We have just shown that the length complexity of $\CRASP_2$ is uncomputable, which appears to contradict the bound shown by \citet{chen2025nonasymptotic}. 
However, this is because they consider a restricted subset of programs, which they notate as $\CRASP^{2,K,T}$. 
In essence, their restricted $\CRASP$ lacks constant bias terms (e.g. $\countl a -\countl b \leq 4$).
We prove this class is strictly weaker than $\CRASP_2$ (even with bounds on size, precision, and girth).
\begin{theorem}
    Dyck languages of bounded depth are not expressible in $\CRASP^{2,K,T}$, for any $K$ and $T$. 
\end{theorem}
On the other hand, bounded-depth Dyck languages are straightforwardly expressed by $\CRASP_2$ with bounded size, like in \cref{example:dyck}.
The full definition of $\CRASP^{2,K,T}$ and a proof can be found in \cref{app:nalg_crasp_definitions}. 

Similarly, \citet{izzo2025quantitative} also derive exponential bounds for a specific theoretical definition of a two layer transformer. 
This specific definition is based off of the limit transformer defined by \citet{huang2024formalframeworkunderstandinglength} (which can simulate $\CRASP$ programs) except with Lipschitz-continuity enforced in the position-wise functions. 
Since non-Lipschitz-continuity appears essential for simulation of $\CRASP$ programs \citep{yang2024counting,huang2024formalframeworkunderstandinglength}, this suggests that their model is not comparable with $\CRASP$, and thus also does not contradict our results. 
More details can be found in \cref{sec:izzo}.

\section{Positive Fragment of $\CRASP$}
\label{sec:crasp-pos}

While the general class of $\CRASP$ programs has an undecidable language emptiness problem, and thus uncomputable length complexity, we find that a natural subset of $\CRASP$ does in fact admit computable length complexity bounds.
In essence, this restricts $\CRASP$ so that one can only count up to a threshold. 
The logic can be seen as a version of counting $\mathsf{LTL}$ \citep{laroussinie2010counting} while omitting the $\mathbf{until}$ operator.

\begin{definition}\label{def:CRASPpos_syntax}
    The syntax of $\CRASPpos$ is defined:
    \begin{align*}
        \phi & \bnfto \sympred \sigma \mid \lnot \phi_1 \mid \phi_1\land \phi_2 \mid \sum _{t\in \mathcal{T}} \alpha_t t \sim k \\
        t & \bnfto \countl[\phi_1] \mid c &&& 
    \end{align*} 
    where $\alpha_t,k,c\in\N$ and $\mathord\sim\in\{<,\leq,=,\geq,>\}$.
    See \cref{def:TLC_semantics} for the semantics. 
\end{definition}

A $\CRASPpos$ program is said to be \emph{decomposed} if all counting formulas are of the form
$\countl[\phi] \ge c$, where $\phi$ is decomposed and $c \in \N$.
Note that the girth of a decomposed program is at most 1.
\begin{lemma}\label{lem:decomposition}
For every $\CRASPpos$ program of size $n$, depth $d$, precision $p$, and girth $g$ %
there is an equivalent decomposed $\CRASPpos$ program 
of size $O(n\,2^{\poly(p,g,d)})$,
depth $O(d)$, and precision $O(p)$.
\end{lemma}
\begin{proof}
First, note that we can rewrite 
\newcommand{\textto}{\quad\text{to}\quad}
\begin{align*}
    \sum_{i=1}^k \alpha_i t_i > c &\textto \sum_{i=1}^k \alpha_i t_i \geq c+1 \\
    \sum_{i=1}^k \alpha_i t_i = c &\textto \sum_{i=1}^k \alpha_i t_i \geq c \wedge \neg\left(\sum_{i=1}^k \alpha_i t_i \geq c+1\right) \\
    \sum_{i=1}^k \alpha_i t_i \le c &\textto \neg\left(\sum_{i=1}^k \alpha_i t_i \geq c+1\right) \\
    \sum_{i=1}^k \alpha_i t_i< c &\textto \neg\left(\sum_{i=1}^k \alpha_i t_i \geq c\right).
\end{align*}

In any counting formula $\sum_{i=1}^k \alpha_i t_i \geq c$, since we can remove any summand $0 \cdot t_i$, we may assume without loss of generality that $\alpha_i>0$ for all $i$.
Let 
\[S := \{(c_1,\dots,c_k) \in [0,c]^k : \alpha_1 c_1 + \dots + \alpha_k c_k \ge c\}.\]
We now write $\sum_{i=1}^k \alpha_i t_i \geq c$ as
$\bigvee_{(c_1,\dots,c_k) \in S} \bigwedge_{i = 1}^k t_i \ge c_i$ and
replace every $t_i \ge c_i$ where $t_i$ is a constant with $\top$ or $\bot$ depending on whether the inequality is satisfied.
Note that the size of $S$ is bounded by $(c+1)^k$, i.e., exponential in the precision $p$ and girth $g$.
Thus, the size of the resulting program is $O(n 2^{\poly(p,g,d)})$.
\end{proof}

\begin{definition}\label{def:TLH_syntax}
    The syntax of $\TLP$ is defined:
    \begin{align*}
        \phi & \bnfto \sympred \sigma \mid \lnot \phi_1 \mid \phi_1\land \phi_2 \mid \Psop \phi \mid \Hsop \phi
    \end{align*} 
    See \cref{def:TLC_semantics} for the semantics. 
\end{definition}

\begin{proposition} \label{thm:crasppos_to_tlh}
For every $\CRASPpos$ program of size $n$, depth $d$, precision $p$, and girth $g$ 
there exists an equivalent $\TLP$ program of size $O(n\, 2^{\poly(p,g,d)})$.
\end{proposition}
\begin{proof}
For the translation from $\CRASPpos$ to $\TLP$ we use \cref{lem:decomposition} to assume a decomposed $\CRASPpos$ program $P$.
Since atomic formulas and Boolean operations can be directly expressed, it suffices to consider formulas of the form $\countl[Q] \ge c$, 
where by induction we assume that we already defined a $\TLP$ formula $\psi$ such that
$\str{w},i \models \psi$ iff $\str{w},i \models Q$ for all $\str{w} \in \Sigma^*$ and $i \in [|\str{w}|]$.
We then define
\begin{align*}
\varphi := &\ \Psop (\psi \wedge \Psop (\psi \wedge \dots \Psop(\psi \land \Psop \psi) \dots) ) \ \vee\\
&\ \psi \wedge \Psop (\psi \wedge \Psop (\psi \wedge \dots \Psop(\psi \land \Psop \psi) \dots) )
\end{align*}
where the $\Psop$ nesting depth in the first row is equal to $c$ and in the second row equal to $c-1$.
(But if $c = 0$, we set $\varphi := \top$.)
The second row is needed since $\countl[Q]$ includes the current position in the counting,
whereas $\Psop \psi$ is interpreted as strictly in the past.
Note that $|\varphi| \in O(c \cdot |\psi|)$, that is, exponential in the number of bits needed to encode $c$.
Thus, together with \cref{lem:decomposition} this yields a $\TLP$ program of size $O(n\,2^{\poly(p,g,d)})$ in total.
\end{proof}

\begin{lemma}\label{lem:tl-model}
For every $\TLP$ program $P$ with $\lang(P) \ne \emptyset$ there exists a string $\str{w} \in \lang(P)$ of length at most polynomial in the size of $P$.
\end{lemma}
\begin{proof}
    Let $P$ be the $\TLP$ program $P_1 := \phi_1, \dots, P_k := \phi_k$ over alphabet $\Sigma$.
We first construct a single formula (not in DAG representation)
\[\psi := \Hsop\big(\big(P_1 \leftrightarrow \phi_1) \wedge \dots \wedge (P_{k-1} \leftrightarrow \phi_{k-1})\big) \wedge \phi_k\]
where we slightly adapt the semantics so that it accepts strings of sets of propositions $\Gamma := \Sigma \cup \{P_1,\dots,P_k\}$.
More precisely, $w',i \models \gamma$ iff $\gamma \in w'_i$ for $\gamma \in \Gamma$, $w' \in (2^\Gamma)^*$, and $i \in [|w'|]$.
It is shown in \citep{ETESSAMI2002279} that $\TLP$ formulas with this adapted semantics have the property that 
if the accepted language is nonempty, then there is an accepted string of length at most polynomial in the size of the formula.
That is, under the assumption that the language of $\psi$ is nonempty, there is a string $w' \in (2^\Gamma)^*$ whose length is at most polynomial in the size of $\psi$.
Observe that by definition of $\psi$, there is a word $w \in \Sigma^*$ such that $w_i \in w'_i$ for all $i$.
Thus, $w$ is a word accepted by $P$ of length at most polynomial in the size of $P$.

\end{proof}

We use this lemma to prove length complexity bounds.

\begin{proposition}\label{prop:crasp-model}
For every $\CRASPpos$ program $P$ with $\lang(P) \ne \emptyset$ there exists a string $\str{w} \in \lang(P)$ of length at most exponential in the size of $P$.
This bound is in fact optimal in the worst case.
\end{proposition}
\begin{proof}
Given a $\CRASPpos$ program $P$ with $\lang(P) \ne \emptyset$, we first apply \cref{thm:crasppos_to_tlh} to obtain a $\TLP$ program $P'$ 
with $\lang(P') = \lang(P)$ of size at most exponential in the size of $P$.
By \cref{lem:tl-model}, if $\lang(P') \ne \emptyset$, there exists a string of length at most polynomial in the size of $P'$,
i.e., exponential in the size of $P$, as required.

To see that the exponential bound is optimal in the worst case,
consider the infinite family $\{P_n\}_{n \ge 1}$ of $\CRASPpos$ programs $P_n := (\countl[a] = n)$ over the alphabet $\Sigma = \{a\}$, where the number $n$ is encoded in binary.
We observe that the smallest (and only) string contained in $\lang(P_n)$ has length $n$, which is exponential in the number of bits needed to encode $n$, so also exponential in the size of~$P_n$.
\end{proof}

\begin{theorem}\label{cor:exp-len-comp}
The length complexity of a $\CRASPpos$ program is exponential in the size of the program.
\end{theorem}
\begin{proof}
Given $\CRASPpos$ programs $P$ and $P'$,
the $\CRASPpos$ program $D := (P \wedge \neg P') \vee (\neg P \wedge P')$
recognizes exactly the set of strings that distinguish $P$ and $P'$.
By \cref{prop:crasp-model}, the smallest string in $\lang(D)$, if nonempty, is of length at most exponential in the size of $D$.
Thus, if $\lang(P) \ne \lang(P')$, there exists a distinguishing string of length at most exponential in the size of $P$ and $P'$. 
\end{proof}

\section{Implications for Transformers}
\label{sec:transformers}

The previous sections discussed length complexity results for $\CRASP$ and $\CRASPpos$, which we connect to transformers here.
We consider transformers that round numbers to a fixed number of bits of precision, but in two slightly different ways. A \emph{\rtfr} does not round inside self-attention, while a \emph{\frtfr} rounds even inside self-attention. 
Intuitively, the distinction boils down to being able to attend uniformly to every single position or only being able to attend to a fixed number of positions. 
Please see \cref{sec:transformer_definitions} for precise definitions.

For the first main result, because \rtfrs{} and $\CRASP$ programs define the exact same class of languages (cf.~\citealp[Theorem~3.1]{yang2025kneedeep}), the uncomputable length complexity bound from \cref{sec:undecidable} applies directly to these transformers. 

\begin{theorem} %
    \Rtfrs{} (even with two layers) do not admit non-asymptotic length generalization.
    \label{th:trans-unbounded}
\end{theorem}

This theorem implies that no learning algorithm can decide if a \rtfr{} has seen enough data.
The same bound applies to any class of transformers which are known to contain $\CRASP$. 
For instance, transformers with polynomial padding tokens, which subsume $\mathsf{FO}$-uniform $\mathsf{TC}^0$ \citep{merrill2025exact}, transformers with polynomial temperature scaling \citep{yang-etal-2025-simulating}, and limit transformers \citep{huang2024formalframeworkunderstandinglength}.

In contrast, by bounding the precision within attention, we find that \frtfrs{} do 
admit length generalization.
\begin{theorem} \label{thm:frtfr_lg}
    \Frtfrs{} admit length generalization with exponential length complexity.
\end{theorem}

This is shown using the expressive equivalence between \frtfrs{} and $\CRASPpos$, which is a corollary of the
equivalence between \frtfrs{} and $\TLP$ \citep[cf.][]{li2025characterizing},
and \cref{thm:crasppos_to_tlh}. 
We provide a direct translation into $\CRASPpos$ in \cref{app:size_bounds}, and derive upper bounds on its size, depth, precision, and girth.

First, we show a singly exponential length complexity lower bound for \frtfrs{} via a polynomial sized translation from $\CRASPpos$ to \frtfrs{}.

\begin{proposition}
    The length complexity of \frtfrs{} with precision $p$, dimension $d$, and depth $L$ is $\Omega(2^{\poly(p,d,L)})$. 
\end{proposition}
\begin{proof}
    The length complexity of $\CRASPpos$ is exponential (\cref{cor:exp-len-comp}). By \cref{thm:crasppos_to_transformers} there is a polynomial transformation from $\CRASPpos$ programs into transformers, attaining the bound.
\end{proof}

Intuitively, this means that any learning algorithm would need to check strings of exponential length before being able to identify the ground-truth solution which is able to length generalize.
Consider the witnessing language $\countl[a]\geq 2^p$ (all strings with at least $2^p$ $a$'s). 
A $1$-layer \frtfr{} with precision $p$ can recognize this language, and this transformer only accepts strings of exponential length $\geq 2^p$. 

For the matching singly exponential upper bound, the translation shown in \cref{app:transformer_to_logic} produces a $\CRASPpos$ program with exponential size, but linear depth, precision, and girth.
Thus, the translation into $\TLP$ shown in \cref{thm:crasppos_to_tlh} does not incur an additional blowup, and the entire program will have exponential size. 

\begin{proposition}
    \Frtfrs{} can be converted to exponentially large equivalent 
    $\TLP$ formulas.
\end{proposition}

Whether or not there exists a polynomial size translation from \frtfrs{} into $\CRASPpos$, we leave as an open question.
Then, using this proposition we can derive the upper bound.  

\begin{proposition}
    The length complexity of \frtfrs{} with precision $p$, dimension $d$, and depth $L$ is $O(2^{\poly(p,d,L)})$. 
\end{proposition}
\begin{proof}
    By the previous proposition, we can find an equivalent $\TLP$ program of exponential size. 
    By \cref{lem:tl-model}, we only need to check strings of length polynomial in the program size, and thus exponential in the transformer size.  
\end{proof}

\section{Discussion}

We have shown that there is no algorithm for learning length generalizing transformers in the general case, and any algorithm for \frtfrs{} must train on strings of exponential length.  
This was shown by analyzing the complexity of the language emptiness problem for $\CRASP$ and $\CRASPpos$, in order to show that length complexity was uncomputable and exponential, respectively.

Our results provide a perspective on the observed difficulty of length generalization in transformers. 
First, length generalization in practice is often quite sensitive to the initialization of the model, learning rate, and other intricacies \citep{zhou2024transformers}. 
Furthermore, even in controlled experiments, length generalization is only ever observed to be partial. 
For instance, transformers may exhibit generalization from lengths 50 to 150 \citep{huang2024formalframeworkunderstandinglength}, or even 40 to 500 \citep{li2025characterizing,deletang2023neural}, but performance inevitably degrades at long enough lengths. 
In these experiments, these transformers have enough depth, width, and precision to express the ground truth solution, so any deficiencies must be as a result of learning dynamics.
Our results give one reason for these failures: any learning algorithm may need to see unfeasibly long strings in order to learn a perfectly length-generalizing transformer. 

We have seen that length-generalization as an empirical phenomenon does not follow the typical scaling laws in machine learning, and logical complexity provides a useful perspective. 
In particular, scaling model size and data does not necessarily help , while expressivity in $\CRASP$ accurately predicts length generalization \citep{huang2024formalframeworkunderstandinglength}.
This insight is not novel, but our bounds bring this insight closer to practice by proving the first length-complexity bounds for entire classes of transformers, rather than restricted sub-classes.

An important future research direction is to identify other fragments of C-RASP and limit transformers, for which computable length generalization bounds exist. In addition, we believe that a more fine-grained analysis of the sample size required for length generalization --- although it has not been much addressed for transformer languages --- is important. For example, although $\CRASPpos$ has an exponential length generalization bound, it is possible that the required sample size is still of a polynomial size. We leave this as future work.

\section*{Impact Statement}
This paper presents fundamental results in  machine learning. Although there may be potential societal consequences of our work, none are direct enough to be specifically highlighted here. 

\section*{Acknowledgements}

We thank Aarohi Srivastava and Katsumi Ibaraki for helpful feedback on writing, as well as Dana Angluin and Michael Hahn for fruitful discussion. 
We thank the anonymous reviewers for their helpful comments.
This material is based in part upon work supported by Deutsche 
Forschungsgemeinschaft (grant number 522843867), the European Union
\includegraphics[width=0.75cm]{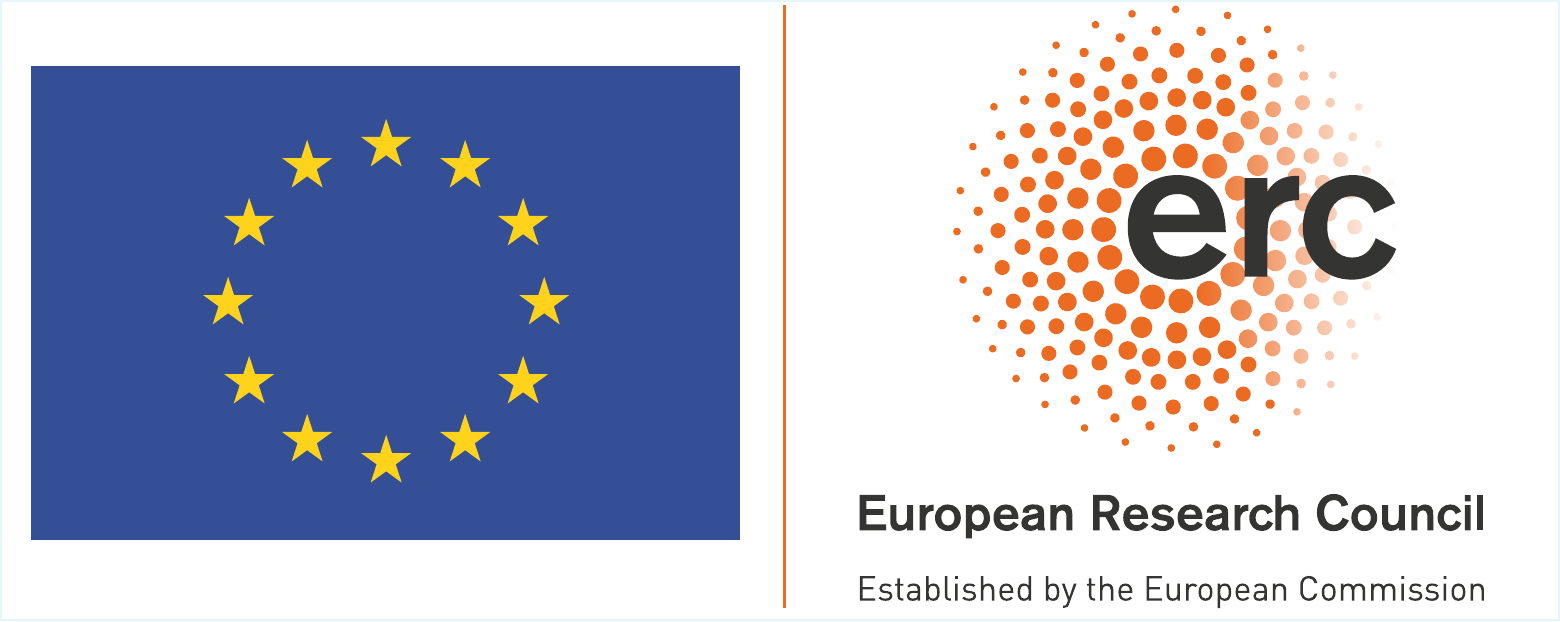}
(ERC, LASD, 101089343, \url{https://doi.org/10.3030/101089343} and FINABIS, 101077902, \url{https://doi.org/10.3030/101077902}) and the US 
National Science Foundation (grant number~2502292).
Andy Yang is supported by the US National Science Foundation Graduate Research Fellowship Program under Grant No.~2236418.

Views and opinions expressed are those of the author(s) only and do not necessarily reflect those of the European Union or the European Research Council Executive Agency. Neither the European Union nor the granting authority can be held responsible for them.

\bibliographystyle{icml2026}
\bibliography{ref}
\appendix
\onecolumn

\section{Relationship to Previous Length Complexity Bounds}

\subsection{\protect\citet{chen2025nonasymptotic}}
\label{app:nalg_crasp_definitions}

While \citeauthor{chen2025nonasymptotic}'s $\CRASP^{2,K,T}$ has a length complexity of $T^{O(K)}$, we show that $\CRASP_2$ has uncomputable length complexity. 
Here, we show why this is not a contradiction.
We will show that $\CRASP^{2,K,T}$ is strictly distinct from the full class of depth $2$ programs $\CRASP_2$. 
The subset $\CRASP^{2,K,T}$  contains $\CRASP$ programs of the following form:
\begin{definition}\citep[Definition~5.4]{chen2025nonasymptotic}\label{def:crasptwo}
    The class $\CRASP^{2,K,T}$ contains programs over alphabet $\Sigma=\{a,b\}$ parameterized by $0\leq z\leq T$ and distinct $\alpha_\psi,\beta_\psi,\lambda_\psi\in[-T,T]$ for $\psi\in\Psi=\{\psi_1,\psi_2,\ldots,\psi_K\}$, where $\frac{\alpha_\psi}{\beta_\psi}\in(0,1)$, $1\leq K\leq T^2$, and $\sum_{\psi\in\Psi} \lambda_\psi>z$, and which implement the formula
    \[\phi:=\sum_{\psi\in\Psi} \lambda_\psi\countl\psi>z\countl\top \]
    where each $\psi$ is of the form
    \[\psi:=\alpha_\psi\countl a > \beta_\psi\countl\top.\]
\end{definition}

First, we define Dyck languages of bounded depth, which will be the separating language class: 

\begin{definition}
    Define $\dyck_k$, the Dyck language of bounded depth $k$, as the language of strings of the following form:
    \begin{align*}
        \dyck_1&=(ab)^*\\
        \dyck_{k+1}&=(a\dyck_{k}b)^*
    \end{align*}
    In other words, strings in $\dyck_k$ satisfy $0\leq \countl a-\countl b\leq k$ everywhere, and satisfy $\countl a=\countl b$ at the end of the string.
\end{definition}

We prove that bounded Dyck languages are not expressible in $\CRASP^{2,K,T}$ using a geometric argument.
Intuitively, we show that all formulas of $\CRASP^{2,K,T}$ will eventually output a constant value for enough strings from any $\dyck_k$ for every $k$.
Then since $\dyck_{k+1}$ contains both strings in $\dyck_k$ and outside of $\dyck_k$, this shows that no such program can recognize any $\dyck_k$.

First we state a lemma, from which the theorem follows. 

\begin{lemma}\label{lem:crasp_band_constant}
    Let $\Psi$ be a set of $\CRASP$ programs $\psi$ of the form $\alpha_\psi\countl a > \beta_\psi\countl\top$.  
    There exists a constant $\epsilon_{k,\Psi}$ such that for all $\psi\in\Psi$ and all $\str{w}\in\dyck_{k}$, either $(\countl \psi)^{\str{w},i}=i\pm\epsilon_{k,\Psi}$ for all $i$, or $(\countl \psi)^{\str{w},i}=\pm\epsilon_{k,\Psi}$ for all $i$. 
    
\end{lemma}
\begin{proof}
    This can be seen through a geometric interpretation of $\CRASP$ formulas. 
    We can plot the number of $a$'s and $b$'s in each word on a 2d grid, where the $x$-axis denotes $\countl a$'s and the $y$-axis denotes $\countl b$. 
    Every string can be viewed as a path through the grid, where steps to the right denote $a$'s and steps up denote $b$'s. 
    
    Then, each $\phi$ defines a line $\ell_\phi$ demarcating a half-plane, such that the path of a word lands in a half-plane iff that word satisfies $\phi$. 
    Visually, we can think of $\dyck_k$ as the set of strings contained within the rectangular band coming out of the origin with width $k$.
    That is, the set of strings which correspond to paths that start and end on the line $\countl[a]=\countl[b]$ without ever leaving the band. 
    
    Let $\ell_1$ be the line with least slope greater than $1$ (or the $y$-axis), and let $\ell_2$ be the line with greatest slope less than or equal to $1$ (or the $x$-axis). 
    \cref{fig:dyck_diagram} visualizes $\dyck_k, \ell_1$, and $\ell_2$. 
    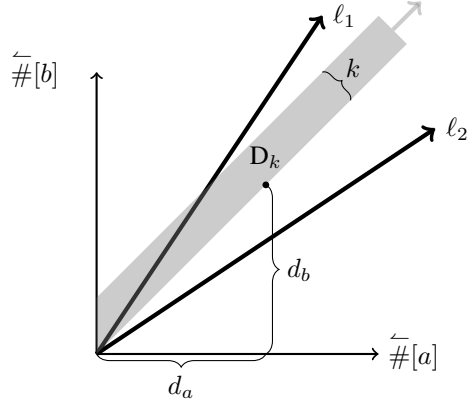
\begin{figure}
        \begin{center}
    
    \begin{tikzpicture}[x=0.75cm,y=0.75cm]
        
        Draw axes with number ticks
        \draw[thick,->] (0,0) -- (0,5);
        \node[left] at (-0.5,5) {$\countl[\sympred{b}]$}; 
        
        \draw[thick,->] (0,0) -- (5,0) node [right] {$\countl[\sympred{a}]$}; %

        \draw[ultra thick, ->] (0,0) -- (4,6) node [right] {$\ell_1$}; 
        \draw[ultra thick, ->] (0,0) -- (6,4) node [right] {$\ell_2$}; 

        \fill[black!50, opacity=0.4] (0,0) -- (0,1) -- (5,6) -- (5.5,5.5) -- (0,0) -- cycle;

        \node at (3,3.5) {$\dyck_k$};
        \draw[->, black!50, opacity=0.4, ultra thick] (5.25,5.75) -- (5.75,6.25);
        \draw[decorate,decoration={brace,amplitude=3pt}] (4,5) -- (4.5,4.5) node [midway, above right] {$k$};
        \node[fill, circle, black, inner sep=0.2ex] at (3,3) {};

        \draw[decorate,decoration={brace,amplitude=5pt,mirror}] (0,0) -- (3,0) node[midway,below=1ex] {$d_a$};
        \draw[decorate,decoration={brace,amplitude=5pt,mirror}] (3,0) -- (3,3) node[midway,right=1ex] {$d_b$};
    \end{tikzpicture}
\end{center}
        \caption{A visualization of the proof idea. $\ell_1$ and $\ell_2$ are the two lines above and below the line $\countl a =\countl b$, the horizontal band contains all strings in $\dyck_k$, and $\epsilon_{k,\Psi}=d_a+d_b$.}
        \label{fig:dyck_diagram}
    \end{figure}

    Because $\ell_1$ and $\ell_2$ are distinct lines coming out of the origin, the distance between them increases without bound the farther we get from the origin. 
    Then there is some length $\epsilon_{k,\Psi}$ such that any strings in $\dyck_k$ with length greater than $\epsilon_{k,\Psi}$ will have paths that do not cross $\ell_1$ or $\ell_2$. 
    The truth values of each $\psi\in\Psi$ will then be fixed to a constant $\top$ or $\bot$ for all such strings.
\end{proof}

\begin{theorem}
    For all $k$, $\dyck_k$ is not expressible in $\CRASP^{2,K,T}$ for any $K,T$. 
\end{theorem}
\begin{proof} 
     Assume for sake of contradiction that $\phi$ is a $\CRASP^{2,K,T}$ program which recognizes $\dyck_k$
     Let $\Psi$ be the set of all $\psi$ in $\phi$. 
     By \cref{lem:crasp_band_constant}, we can partition $\Psi$ into the set of formulas $\Psi_\top$ (which are satisfied at $|\str{w}|\pm\epsilon_{k,\Psi}$ positions) and $\Psi_\bot$ (which are satisfied at $\pm\epsilon_{k,\Psi}$ positions) over strings in $\dyck_{k+1}$.
     We can plug these values into the definition of $\phi$ to find that the truth value of $\phi$ is equivalent to:
     \[\sum_{i\colon \psi\in\Psi_\top} \lambda_\psi(|\str{w}|\pm \epsilon_{k,\Psi}) + \sum_{i\colon \psi\in\Psi_\bot} \lambda_\psi(\pm \epsilon_{k,\Psi})>z|\str{w}|.\]
     Here, the truth value of $\phi$ only depends on the length of the string $\str{w}\in \dyck_k$ on which it is evaluated. 
     Further grouping terms of the above equation into constants, we can see that the entire formula boils down to the following expression depending only on the length of the string:
     \[C_1|\str{w}|+C_2>z|\str{w}|.\]
     For $\str{w}$ with length greater than some sufficiently large threshold, $C_1|\str{w}|$ and $z|\str{w}|$ vastly outweigh $C_2$, and thus the truth value of $\phi$ becomes constant for any other string beyond the threshold length.
     This implies either $\dyck_{k+1}\subseteq \lang(\phi)$ or $\dyck_{k+1}\subseteq \Sigma^*\setminus \lang(\phi)$, restricted to strings of length beyond some threshold.  
     However, since $\lang(\phi)=\dyck_k$, this implies that either $\dyck_{k+1}\subseteq\dyck_k$ or $\dyck_{k+1}\subseteq \Sigma^*\setminus \dyck_k$, beyond some length threshold. 
     This contradicts the fact that $\dyck_k\subsetneq \dyck_{k+1}$. 
     Thus, no such program $\phi$ exists. 
\end{proof}

As a final note, we observe that \citet{chen2025nonasymptotic} use the absolute size of the constants $T$ as precision, while we use the bit-precision $p$ of the constants (exponentially smaller than the absolute size). 
This accounts for why it is non-contradictory that the depth $1$ $\CRASP$ fragment they study has quadratic length complexity, while we show an exponential lower bound for $\CRASP$ of depth $1$.

\subsection{\protect\citet{izzo2025quantitative}}
\label{sec:izzo}

\Citet{izzo2025quantitative} prove quantitative length generalization bounds for two variants of transformers. First, they consider transformers in which:
\begin{itemize}
\item Parameters and activations are stored with $(1,1)$-precision.
\item There is only $1$ layer.
\item Absolute position encodings (APEs) are periodic.
\item Relative position encodings (RPEs) are added to the attention logits; they must be local, that is, zero when the query and key distance exceeds a maximum distance.
\item The gap between the maximal attention logit and any non-maximal attention logit is guaranteed to be at least $\gamma$.
\item Attention logits are scaled by $\log n$ where $n$ is the input length.
\item Position-wise functions are $1$-Lipschitz continuous.
\end{itemize}
Then they obtain a length generalization bound exponential in $2^{1/\gamma}$ and polynomial in various properties of the transformer.
Our \cref{thm:frtfr_lg} gives a similar exponential bound for \frtfrs, but using a more standard definition of transformer and for an arbitrary number of layers.

Second, \citet{izzo2025quantitative} consider transformers in which:
\begin{itemize}
\item Parameters and activations are stored with $(\infty, \infty)$-precision.
\item There are at most $2$ layers.
\item Periodic APEs and local RPEs as above.
\item The gap between the maximal RPE and any non-maximal RPE is guaranteed to be at least $\gamma$.
\item The RPEs are scaled by $\log i$ where $i$ is the query position.
\item Position-wise functions are $1$-Lipschitz continuous.
\end{itemize}
Then they obtain a length generalization bound exponential in various properties of the transformer.
In the case where the absolute and relative position encodings are set to zero, this is a standard transformer, so their bound would appear to contradict our \cref{th:trans-unbounded}.
The difference seems to be in their assumption of Lipschitz-continuous position-wise functions. Our \cref{th:trans-unbounded} is based on a translation of $\CRASP$ to transformers, and all three such translations that we are aware of depend on some non-Lipschitz position-wise function: layer normalization \citep{yang2024counting}, the Heaviside step activation function \citep{huang2024formalframeworkunderstandinglength}, or rounding \citep{yang2025kneedeep}. Such a function seems critical for testing the equalities in system of a Diophantine equations.

\section{Transformer Size Bounds}\label{app:size_bounds}

\subsection{Definitions}
\label{sec:transformer_definitions}

\def\explogit{A}
\begin{definition} \label{def:fixed_precision}
A \emph{fixed-precision number} with $\numbits$ total bits and $\fracbits$ fractional bits is a rational number of the form $m \cdot 2^{-\fracbits}$ where $m$ is an integer and $-2^{\numbits-1} \le m < 2^{\numbits-1}$. We write $\mathbb{F}_{\numbits,\fracbits}$, or simply $\mathbb{F}$, for the set of all fixed-precision numbers with $p$ total bits and $s$ fractional bits.

We represent negative numbers using two's complement. If $\bitind \in [\numbits]$, the $\bitind$-th bit of a fixed-precision number $x$, written $\bit{x}{\bitind}$, is defined as
\begin{align*}
\bit{x}{\bitind} &= \begin{cases}
1 & \text{if $\lfloor x / 2^{\bitind-\fracbits-1} \rfloor$ is odd} \\
0 & \text{otherwise.}
\end{cases}
\end{align*}
In two's complement we note that 
\begin{equation*}
    x = -2^{\numbits}\cdot\bit{x}{\numbits}+\sum_{\bitind\in[\numbits-1]} 2^{\bitind-\fracbits-1}\cdot \bit{x}{\bitind}.
\end{equation*}

If $x$ is a real number, we write $\round_{\mathbb{F}}(x)$ or simply $\round(x)$ for the greatest element of\/ $\mathbb{F}$ less than or equal to $x$. 
\end{definition}

\begin{definition}\label{def:transformer}
A (future-masked) \emph{\rtfr} of depth $k$ is a function $\tfr \colon \Sigma^* \to \mathbb{F}$, defined in terms of functions
\begin{align*}
E \colon \Sigma &\to \mathbb{F}^d  && \text{word embedding}\\
W_{\textnormal{Q}}^{(\ell)}, W_{\textnormal{K}}^{(\ell)}, W_{\textnormal{V}}^{(\ell)}, \colon \mathbb{F}^d &\to \mathbb{F}^d \qquad \ell = 1, \ldots, k  && \text{self-attention}\\
f^{(\ell)} \colon \mathbb{F}^d &\to \mathbb{F}^d \qquad \ell = 1, \ldots, k  && \text{feed-forward}\\
W_{\textnormal{out}} \colon \mathbb{F}^d &\to \mathbb{F} && \text{output unembedding}.
\end{align*}
On input $\str{w}$, $\tfr(\str{w})$ is computed as follows:
\begin{align}
\mathbf{h}^{(0)}_{i}(\str{w}) &= E(w_i) \label{eq:emb} \\
\intertext{For $\ell = 1, \ldots, k$:}
\mathbf{q}^{(\ell)}_{i}(\str{w}) &= W_{\textnormal{Q}}^{(\ell)}\left(\mathbf{h}^{(\ell-1)}_{i}(\str{w})\right) \\
\mathbf{k}^{(\ell)}_{i}(\str{w}) &= W_{\textnormal{K}}^{(\ell)}\left(\mathbf{h}^{(\ell-1)}_{i}(\str{w})\right) \\
\mathbf{v}^{(\ell)}_{i}(\str{w}) &= W_{\textnormal{V}}^{(\ell)}\left(\mathbf{h}^{(\ell-1)}_{i}(\str{w})\right) \\
\mat{s}^{(\ell)}_{ij}(\str{w}) &= \mathbf{q}^{(\ell)}_{i}(\str{w}) \cdot \mathbf{k}^{(\ell)}_{j}(\str{w}) \label{eq:att_logit} \\
A^{(\ell)}_i(w)& =\sum_{j=1}^{i} \round\mleft(\exp \mleft(\mat{s}^{(\ell)}_{ij}(\str{w})\mright)\mright)\,\mathbf{v}^{(\ell)}_{j}(\str{w})\\
B^{(\ell)}_i(w)&=\sum_{j=1}^{i}\round\mleft(\exp \mleft(\mat{s}^{(\ell)}_{ij}(\str{w})\mright)\mright)\\
\mathbf{c}^{(\ell)}_{i}(\str{w}) &= \round\left(\frac{A^{(\ell)}_i(w)}{B^{(\ell)}_i(w)}\right) \label{eq:att} \\
\intertext{where \cref{eq:att} evaluates to the average of all $\mathbf{v}_j^{(\ell)}$ if the denominator is $0$, and finally}
\mathbf{h}^{(\ell)}_{i}(\str{w}) &= f^{(\ell)}\left(\mathbf{c}^{(\ell)}_{i}(\str{w}) + \mathbf{h}^{(\ell-1)}_i(\str{w})\right) \\
\tfr(\str{w}) &= W_{\textnormal{out}}\left(\mathbf{h}^{(k)}_{|\str{w}|}(\str{w})\right).
\end{align}

We say that $\tfr$ \emph{accepts} $w$ if\/ $\tfr(\str{w})>0$. 

To measure the size of a transformer, we assume that the constants in all of its defining functions are given in binary
and the precision is given in unary.
\end{definition}

Note crucially that \cref{eq:att} is written so that even if $i \gg 2^\fracbits$, it is still possible to obtain nonzero values.

\begin{definition}
A (future-masked) \emph{\frtfr} is defined similarly to a \rtfr{} except the attention weights are rounded before using them to take a weighted average of the value vectors:
\begin{align*} 
    A^{(\ell)}_{i,j}(w)& = \round\mleft(\exp \mleft(\mat{s}^{(\ell)}_{ij}(\str{w})\mright)\mright) \\
    \alpha^{(\ell)}_{i,j}(\str{w}) &= \round\mleft( \frac{A^{(\ell)}_{i,j}(\str{w})}{B^{(\ell)}_i(\str{w})} \mright)\\
    \mathbf{c}^{(\ell)}_{i}(\str{w}) &= \round\left(\sum_{j=1}^{i}
    \alpha^{(\ell)}_{i,j}(\str{w}) \,
    \mathbf{v}^{(\ell)}_{j}(\str{w})\right)
\end{align*}

\end{definition}

We will often make use of the following operator in $\CRASP$, which does not increase its expressive power or affect the depth of formulas, but saves space when writing.
    \begin{align*}
            (\cif{\phi}{t_\text{then}}{t_\text{else}})^{w,i} &= \begin{cases}
                t_\text{then} & w,i\models\phi\\
                t_\text{else} & w,i\not\models\phi\\
            \end{cases}
    \end{align*}
\begin{lemma}[\citealt{yang2024counting}]
    Any formula $\phi$ of\/ $\CRASP$ that uses the $?$ operator can be converted into a formula that does not use the $?$ operator, defines the same language as $\phi$, and has the same depth as $\phi$.
\end{lemma}
\begin{proof}
    Any comparison formula involving the $?$ operator can be written in the form 
    \[ (\cif{\psi_{\text{if}}}{t_{\text{then}}}{t_{\text{else}}}) + \sum_{\termind \in [\numterms]} t_\termind \ge C, \]
    which can be rewritten as
    \[ \left( \psi_{\text{if}} \land t_{\text{then}} + \sum_{\termind \in [\numterms]} t_\termind \ge C \right) \lor \left( \lnot\psi_{\text{if}} \land t_{\text{else}} + \sum_{\termind \in [\numterms]} t_\termind \ge C \right). \]
    This rule can be used iteratively to rewrite all the $?$ operators out of a formula. 
\end{proof}

\subsection{Previous Work}

Previous work by \citet{li2025characterizing} characterized the expressivity of \frtfrs{} in terms of temporal logic, and furthermore provided empirical evidence that transformers length-generalized very well on this class of languages.

\begin{theorem}[{\citealp[Thm.~3.2]{li2025characterizing}}] \label{thm:transformer_to_tlp}
    Every \frtfr{} can be simulated by a $\TLP$ formula.
\end{theorem}

\begin{theorem}[{\citealp[Thm.~3.3]{li2025characterizing}}] \label{thm:tlp_to_transformer}
    Every $\TLP$ formula can be simulated by a \frtfr{} .
\end{theorem}

Their construction did not explicitly compute the size bounds, but a careful reading shows that a singly exponential blowup occurs when simulating \frtfrs{} in $\TLP$, and a polynomial blowup occurs the other way around. 

Our contribution is to derive precise size bounds in converting between $\CRASPpos$ and \frtfrs{}, which also turn out to be singly exponential and polynomial, in the same way. 
However, an additional point we note is that the depth of the $\TLP$ formula turns out to be exponential in the size of the transformer, while the depth of the $\CRASPpos$ formula is linear in the depth of the transformer. 
This perspective may be used to shed more precise insight on the sizes of transformers, similar to how \citet{yang2025kneedeep} used a depth-preserving equivalence between $\CRASP$ and \rtfrs{} in order to derive a strict depth hierarchy for transformers.

\subsection{Fixed-Precision Transformers to Logic}\label{app:transformer_to_logic}

\begin{definition}
Let $\vec{x}_i \colon \Sigma^* \to \F^d$ be a sequence of activation vectors depending on the input string $\str{w}$.
We will often omit the dependence on  $\str{w}$ and just write $\mat{x}_i$ instead of $\mat{x}_i(\str{w})$ whenever it is clear from context. 

We say that a set of $\CRASPpos$ predicates $P_{\bit{\vec{x}_c}{\bitind}}$ (for $c \in [d]$ and $\bitind \in [\numbits]$) simulates $\vec{x}$ if for every $c \in [d]$ and $\bitind\in[\numbits]$,
    \[\str{w},i \models P_{\bit{\vec{x}_c}{\bitind}} \iff \bit{\vec{x}_i(\str{w})_{c}}{\bitind}=1. \]
\end{definition}
We first define what it means for a \rtfr{} and a formula to simulate each other.

\begin{definition}
    We say that a $\CRASP$ formula $\phi$ \emph{simulates} a \rtfr{} $\tfr$ %
    if, for all $w\in \Sigma^*$,
    \begin{equation*}
    w,i\models \phi \iff \tfr(\bos\cdot w)_i>0.
    \end{equation*}
    In other words, $w\models\phi$ if and only if $\tfr$ accepts $\bos \cdot w$.
    
    We say that a \rtfr{} $\tfr$ with 
    depth $k$ and dimension $d$
    \emph{simulates} a formula $\phi$ of \/$\CRASP$ $\tfr$ if, for all $w\in \Sigma^*$, 
    \begin{equation*}
        \tfr(\bos \cdot w)_i>0 \iff w,i\models\phi.
    \end{equation*}
    Again, $\tfr$ accepts $\bos \cdot w$ if and only if $w\models\phi$.
\end{definition}

We first state a utility lemma which we use repeatedly to complete the main theorem. It says that any positionwise function on fixed-precision numbers can be defined in propositional logic.

\begin{lemma}\label{lem:finite_function}
    Let $f\colon\F^{d}\to\F^{d}$ be a function.
    Let $\vec{x}_i\colon\Sigma^*\to\F^d$ be a sequence of activation vectors simulated by predicates $P_{\bit{\vec{x}_c}{\bitind}}$ for $c\in[d]$ and $\bitind\in[\numbits]$. 
    Then there exist $\TLP$ programs 
    \begin{enumerate}[label=(\alph*)]
    \item\label{item:function} $P_{\bit{f(\vec{x})_c}{\bitind}}$ simulating $f(\vec{x}_i)$, with size $O(p^2d^2 2^{pd})$.
    \item\label{item:constant} $P_{\vec{x} = \vec{y}}$, where $\vec{y} \in \F^d$ (not depending on $w$) such that
    \begin{align*}
    \str{w},i\models P_{\vec{x}=\vec{y}} &\iff \vec{x}_{i}(\str{w})=\vec{y}
    \end{align*} 
    with size $pd$.
    \end{enumerate}
\end{lemma}
\begin{proof}
    \labelcref{item:function} For each $c\in[d]$ and $\bitind\in[\numbits]$ we define the program line:
    \begin{align*}
        P_{\bit{f(\vec{x})_c}{\bitind}} &:= \bigvee_{\substack{\vec{v}\in\F^d\\\bit{f(\vec{v})_c}{\bitind}=1}} \bigwedge_{\substack{c'\in[d],\bitind'\in[\numbits]\\\bit{\vec{v}_{c'}}{\bitind'}=1}} P_{\bit{\vec{x}_{c'}}{\bitind'}}.
    \end{align*}
    The size of each such line is at most $pd\,2^{pd}$. 
    As there are $pd$ such lines, this increases the size of the program (additively) by $O(p^2d^22^{pd})$.

    \labelcref{item:constant} Define
    \begin{align*}
        P_{\vec{x}=\vec{y}} &:= \bigwedge_{c\in[d],\bitind\in[\numbits]} \left( 
        P_{\bit{\vec{x}_c}{\bitind}} 
        \leftrightarrow \bit{\vec{y}_c}{\bitind} \right).
    \end{align*}
    The size of this line is $pd$.

    Overall, the size of the program is $O(p^2d^2\,2^{pd})$. 
    Note that this is in $O(2^{\mathsf{poly}(x,y)})$.
    We note this can be generalized to multi-ary functions with still an exponential bound. 
\end{proof}

\begin{theorem}
    For a \frtfr{} $\tfr$ with bit precision $p$, hidden dimension $d$, and depth $L$, there exists a $\CRASPpos$ program $P$ with:
    \begin{itemize}
        \item Size $O(|\Sigma|Lp^2d2^{p^3d})$
        \item Depth $O(L)$
        \item Precision $O(p)$
        \item Girth $O(p)$ 
    \end{itemize}
\end{theorem}

\begin{proof}
    We will inductively construct a $\CRASPpos$ program that simulates simulate $T$, while maintaining an upper bound on the size of the program.
    We will show for every activation $\mat{h}_{i,c}^\ell$ there exists a program line $P_{\langle \mat{h}_{c}^\ell\rangle_b}$ that simulates $\mat{h}_{i,c}^\ell$.
    To make our presentation more concise, we will often write many $\CRASPpos$ definitions condensed as a large formula in a single line. 
    To rewrite it in proper $\CRASPpos$ semantics, we can expand it into a program with size upper bounded by the number of symbols occurring in the formula. 

    The construction proceeds by induction on $\ell$. For $\ell=0$, the transformer consists of just a word embedding layer, which we simulate using the following program:
    \begin{align*}
        P_{\bit{\mat{h}^{(0)}_{c}}{\bitind}}:=\displaystyle\bigvee_{\substack{\sigma\in\Sigma \\ \bit{E(\sigma)_c}{\bitind}=1}} \sympred\sigma.
    \end{align*}
    Here there are $pd$ program lines, each of size $O(|\Sigma|)$. 
    This program has precision $O(1)$, size $P(pd|\Sigma|)$, and depth $O(1)$

    Now suppose the inductive hypothesis holds for layer $L$, and consider layer $(L+1)$.
    That is, first assume we have defined program lines $P_{\bit{\mat{h}^{(L-1)}_c}{\bitind}}$ for each $\bitind\in[\numbits]$ and $c\in[d]$ simulating the activations coming out of the $(\ell-1)$-th layer, with:
    \begin{itemize}
        \item Size $O(|\Sigma|(L-1)p^2d2^{p^3d})$
        \item Depth $O(L-1)$
        \item Precision $O(p)$
        \item Girth $O(p)$
    \end{itemize}
    To simulate the next layer, we want
    program lines $P_{\bit{\mat{q}^{(\numlayers)}_{c}}{\bitind}}$, $P_{\bit{\mat{k}^{(\numlayers)}_{c}}{\bitind}}$, and $P_{\bit{\mat{v}^{(\numlayers)}_{c}}{\bitind}}$ such that:
    \begin{align*}
        \str{w},i\models P_{\bit{\mat{q}^{(\numlayers)}_{c}}{\bitind}}&\iff \bit{\mat{q}^{(\numlayers)}_{i,c}(\str{w})}{\bitind}=1 \label{eq:q} \\
        \str{w},i\models P_{\bit{\mat{k}^{(\numlayers)}_{c}}{\bitind}}&\iff \bit{\mat{k}^{(\numlayers)}_{i,c}(\str{w})}{\bitind}=1\\
        \str{w},i\models P_{\bit{\mat{v}^{(\numlayers)}_{c}}{\bitind}} &\iff \bit{\mat{v}^{(\numlayers)}_{i,c}(\str{w})}{\bitind}=1.
    \end{align*}

    This allows us to simulate the attention weight matrices. 
    These program lines can be defined using \cref{lem:finite_function}, adding $O(p^2d^22^{p^2d})$ to the program size but not changing the depth, girth, or precision.

    Next, we want to compute the summands $A$ and $B$ in the numerator and denominator of attention. 
    One potential difficulty is that summands retrieve values from two positions ($i$ and $j$) simultaneously, whereas a $\CRASPpos$ program can retrieve a value from one position at a time. 
    However, since $\mat{q}_i$ can only take on finitely many values, we can simulate retrieval by enumerating all these possible values. 

    To perform this enumeration we write program lines $P_{\mat{q}^{(\numlayers)}=\vec{q}}$, $P_{\mat{k}^{(\numlayers)}=\vec{k}}$, and $P_{\mat{v}^{(\numlayers)}=\vec{v}}$ for each $\vec{q},\vec{k},\vec{v} \in \mathbb{F}^d$ (not depending on $w$) such that:
    \begin{align*}
       \str{w},i\models P_{\mat{q}^{(\numlayers)}=\vec{q}} &\iff \mat{q}_i^{(\numlayers)}(\str{w})=\vec{q}\\
        w,j\models P_{\mat{k}^{(\numlayers)}=\vec{k}} &\iff \mat{k}_j^{(\numlayers)}(\str{w})=\vec{k}\\
        w,j\models P_{\mat{v}^{(\numlayers)}=\vec{v}} &\iff \mat{v}_j^{(\numlayers)}(\str{w})=\vec{v}\\
    \end{align*}
    This is also achieved using \cref{lem:finite_function}, adding $O(p^2d^22^{p^2d})$ to the program size but not changing the depth, girth, or precision.
    Then, we compute $\mat{s}^{(\numlayers)}_{ij}(\str{w})$, found in the numerator of the attention scores. 
    This depends on the value of $\vec{q}_i$, so for each $\vec{q}\in\F^d$ we will define a program line $\explogit^{(\numlayers)}_{\vec{q},\bitind}$ such that:
     \begin{align*}
        \str{w},j\models \explogit^{(\numlayers)}_{\vec{q},\bitind} &\iff \bigbit{\round\mleft(\exp \mleft(\vec{q} \cdot \mathbf{k}^{(\numlayers)}_{j}(\str{w})\mright) \mright)}{\bitind}=1.
    \end{align*}
    We write these formulas as
    \begin{align*}
         \explogit^{(\numlayers)}_{\vec{q},\bitind} &:=  \bigvee_{\substack{\vec{k}\in\mathbb{F}^d\colon\\\bigbit{\round\mleft(\exp \mleft(\vec{q} \cdot \mathbf{k}\mright) \mright)}{\bitind}=1}} P_{\mat{k}^{(\numlayers)}=\vec{k}}\\
    \end{align*}
    These add $O(2^{pd})$ to the program size but don't change the depth, girth, or precision.
    
    Then we use this to compute the summand $B^{(\numlayers)}_i(\str{w})=\sum_{j=1}^{i} \round\mleft(\exp \mleft(\mat{s}^{(\numlayers)}_{ij}(\str{w})\mright)\mright)$.
    The idea is to fix $\mat{q}_i^{(\numlayers)}$ for each  $\mat{q}\in\F^d$ and compute what $B^{(\numlayers)}_i(\str{w})$ would be as a counting term, and then perform a disjunction over all $\mat{q}$ to truncate these values back to fixed precision. 
    First, for each $\mat{q}\in\F^d$ we can fix $\mat{q}_i^{(\numlayers)}$ and compute the value as a term $C_{B^{(\numlayers)}_{\mat{q}}}\in\N$ such that:
    \begin{align*}
        C_{\mat{q},B^{(\numlayers)}}^{w,i}&=2^{\fracbits}\sum_{j\leq i} \round\left(\exp\left(\mat{q} \cdot \mat{k}^{(\numlayers)}_j(\str{w})\right)\right).
    \end{align*}
    Because each term is always nonnegative, we can ignore the most significant bit and write the counting term as follows:
    \begin{align*}
    C_{\mat{q},B^{(\numlayers)}} &:=\sum_{\bitind\in[\numbits-1]}2^{\bitind+\fracbits-1}\cdot\countl\left[\explogit^{(\numlayers)}_{\vec{q},\bitind}\right]
    \end{align*}

    Then, we perform a disjunction over all $\mat{q}$ to compute $B^{(\numlayers)}_i(\str{w})$ and then round to fixed-precision. 
    Note that we truncate any values above the maximum fixed-precision number, as per the rounding rules defined. 
    The program line is as follows:
    \begin{align*}
    B^{(\numlayers)}_{\bitind} &:= \bigvee_{\mat{q}\in\F^d}\left(P_{\mat{q}^{(\numlayers)}=\mat{q}}\left(\bigvee_{\substack{x\in\mathbb{F} \\ x<2^{\numbits-s-1} \\ \bit{x}{\bitind}=1}} 
         (C_{\mat{q},B^{(\numlayers)}}=2^s\cdot x) \right)
         \lor 
         (C_{\mat{q},B^{(\numlayers)}}\geq 2^s\cdot x)\right).
    \end{align*}

    Each $C_{\mat{q},B^{(\numlayers)}}$ has size $O(p)$, precision $O(p)$, girth $O(p)$, and depth $O(1)$. 
    Thus in total all the $B^{(\numlayers)}_{\bitind}$ program lines contribute size $O(p^22^{p^2d})$, precision $O(p)$, girth $O(p)$, and depth $O(1)$.  
    
    The next step is to simulate $\alpha_{ij}$.
    To begin, using \cref{lem:finite_function} we can write program lines $P_{B^{(\numlayers)}=v}$ for $v\in\F$ such that 

    \begin{align*}
        w,i\models P_{B^{(\numlayers)}=v} \iff B^{(\numlayers)}=v
    \end{align*}

    This adds $O(p^2d^22^{p^2d})$ to the program size without changing the depth, girth, or precision.
    To remove the dependence of $\alpha_{ij}$ on multiple positions, we again need to enumerate over all $\mat{q}\in\F^d$. 
    We use the notation $\alpha_{\vec{q},j}=\round\left(\frac{\round\mleft(\exp \mleft(\vec{q} \cdot \mathbf{k}^{(\numlayers)}_{j}(\str{w})\mright) \mright)}{\sum_{j=1}^{i} \round\mleft(\exp \mleft(\vec{q} \cdot \mathbf{k}^{(\numlayers)}_{j}(\str{w})\mright) \mright)}\right)$.
    Then for each $\vec{q}$ we can write formulas 
    $P_{\bit{\alpha_{\vec{q}}}{\bitind}}$ such that:
    \begin{align*}
        \str{w},j\models P_{\bit{\alpha_{\vec{q}}}{\bitind}} &\iff \left\langle\alpha_{\vec{q},j}(\str{w})\right\rangle_{\bitind}=1.
    \end{align*}
    
    This uses the multi-ary generalization of \cref{lem:finite_function}, where $\alpha_{\vec{q},j}=f(A^{(\numlayers)}_{\mat{q},\bitind},B^{(\numlayers)}_{\bitind})$.
    This adds $O(p^2d^22^{p^2d})$ to the program size without changing the depth, girth, or precision.
    
    Now that the attention scores have been simulated, the next step is to simulate the weighted average of the values by the attention scores. 
    As observed by \citet{li2025characterizing} and \citet{merrill2023a}, there are finitely many positions where this can be nonzero.
    We label this bound $M^{(\numlayers)}$.
    For this we want to write a count $C_{\mat{q},M^{(\numlayers)}}$ such that 
    \[\left(C_{\mat{q},M^{(\numlayers)}}\right)^{\str{w},i}=\left|\left\{j\in[|\str{w}|]\colon \alpha_{\mat{q}j}>0\right\}\right|\]

    First, we write a formula $P_{\alpha_{\vec{q}}>0}$ for each $\vec{q}\in\F^d$ such that:
    \[w,j\models P_{\alpha_{\vec{q}}>0} \iff \alpha_{\vec{q},j}(\str{w})>0 \]
    
   This can be achieved using \cref{lem:finite_function},  $O(p^2d^22^{p^2d})$ to the program size without changing the depth, girth, or precision.
   Then we write a count term 

    \[C_{\mat{q},M^{(\numlayers)}}:=\countl\left[P_{\alpha_{\vec{q}}>0}\right]\]

    The exact value of $C_{\mat{q},M^{(\numlayers)}}$ must range from $0$ to $2^s$ \citep[Proposition~B.6]{li2025characterizing}. 
    So there exist formulas $P_{M^{(\numlayers)}=M}$ for $M\in [0,2^s]$ such that 

    \[\str{w},i\models P_{M^{(\numlayers)}=M} \iff M^{(\numlayers)} = M.\]

    We can then write this as follows:

    \begin{align*}
        P_{M^{(\numlayers)}=M} := \bigvee_{\mat{q}\in\F^d} C_{\mat{q},M^{(\numlayers)}}=M.
    \end{align*}

    Each $C_{\mat{q},M^{(\numlayers)}}$ has size $O(1)$, girth $O(1)$, depth $(1)$, and precision $O(1)$. 
    When plugged into each $P_{M^{(\numlayers)}=M}$, this results in an overall contribution of size $O(2^{pd})$, precision $O(1)$, depth $O(1)$, and girth $O(1)$. 
    
    Next, we need to check if $\round\left(\sum_{j=1}^{i}\alpha_{\vec{q},j}\mathbf{v}^{(\numlayers)}_{j}(\str{w})\right) = \vec{x}$ for $\vec{x}\in \F^d$. 
    The reason this is not straightforward is that we may have to compute a summation of negative-valued $\mat{v}$ (which is not definable in $\CRASPpos$).
    
    The solution we use is to turn the summation nonnegative by adding $2^{\numbits}$ to any nonzero summand. 
    Since there are exactly $M^{(\numlayers)}$ positions receiving nonzero attention, we add exactly $M^{(\numlayers)}\cdot2^{\numbits}$ to the total. 
    Stated more formally, we simulate the output of attention using only the summation of nonnegative values:
    \[\round\left(\sum_{1\leq j\leq i}^{i}\alpha_{\vec{q},j}\left(\mathbf{v}^{(\numlayers)}_{j}(\str{w})_c\right)\right) = \vec{x}_c\iff \round\left(\sum_{\substack{1\leq j\leq i\colon\\\alpha_{\mat{q}j}>0}}\alpha_{\vec{q},j}\left(\mathbf{v}^{(\numlayers)}_{j}(\str{w})_c\right)+2^{\numbits}\right) = \vec{x}_c+M^{(\numlayers)}_{\vec{q}}\cdot2^{\numbits}.\]
   
    To accomplish this, first we write program lines $P_{\bit{\alpha_{\mat{q}}\mat{v}^{(\numlayers)}_{c}}{\bitind}}$ to define the bits of $\bit{\alpha_{\mat{q}j}\mat{v}^{(\numlayers)}_{j,c}}{\bitind}$. 
    That is:
    \begin{align*}
        \str{w},j\models P_{\bit{\alpha_{\mat{q}}\mat{v}^{(\numlayers)}_{c}}{\bitind}} \iff \bit{\alpha_{\mat{q}j}\mat{v}^{(\numlayers)}_{j,c}}{\bitind}=1
    \end{align*}
    This is achieved again using \cref{lem:finite_function}, again adding $O(p^2d^22^{p^2d})$ to the program size without changing the depth, girth, or precision.

    Then, for the summation, we can sum the first $\numbits-1$ bits as usual (because they are all positive). 
    Because we add a $2^{\numbits}$ whenever $\alpha_{\mat{q}j}\mat{v}^{(\numlayers)}_{j,c}$, this interacts with the summation of $-2^{\numbits}$ in $\numbits$-th bit. 
    Specifically, to compute the sum:
    \begin{itemize}[noitemsep]
        \item If $\alpha_{\vec{q}}=0$, we add $0$ 
        \item If $\alpha_{\vec{q}}>0$ and $\bit{\alpha_{\mat{q}j}\mat{v}^{(\numlayers)}_{j,c}}{\numbits}=1$, we add $0$
        \item If $\alpha_{\vec{q}}>0$ and $\bit{\alpha_{\mat{q}j}\mat{v}^{(\numlayers)}_{j,c}}{\numbits}=0$, we add $2^p$
    \end{itemize}
    Thus we just need to check $P_{\alpha_{\vec{q}}>0}\land \lnot P_{\bit{\alpha_{\mat{q}}\mat{v}^{(\numlayers)}_{c}}{\bitind}}$ to determine how to perform the summation in the $\numbits$-th bit.

    Thus we can compute the modified sum as follows:

    \begin{align*}
        C_{\mat{q},\mat{c}^{(\numlayers)}_c} := 2^{\numbits}\cdot\countl\left[P_{\alpha_{\vec{q}}>0}\land \lnot P_{\bit{\alpha_{\mat{q}}\mat{v}^{(\numlayers)}_{c}}{\bitind}}\right]+\sum_{\bitind\in[\numbits-1]} 2^{\bitind-\fracbits-1}\cdot \countl\left[P_{\bit{\alpha_{\mat{q}}\mat{v}^{(\numlayers)}_{c}}{\bitind}}\right]
    \end{align*}

    Finally we round to fixed precision. For $\vec{x}\in\F$ we write a program line checking if $\mat{c}_c^{(\numlayers)}=\vec{x}$ as follows. 
    Crucially, we know there are exactly $M$ positions where $\alpha_{\mat{q}j}\mat{v}^{(\numlayers)}_{j,c}$ is not zero in every bit.
    So we can check the value of $P_{M_{\vec{q}=M}}$ and use that to add the correction factor to the sum.

    \begin{align*}
    P_{\mat{c}_c^{(\numlayers)}=\vec{x}} &:= \bigvee_{\mat{q}\in\F^d} P_{\mat{q}^{(\numlayers)}=\mat{q}} \left(\bigvee_{M\in\F_{s,0}} \left(P_{M_{\vec{q}=M}} \land \left(\bigvee_{\substack{x\in\mathbb{F} \\ x<2^{\numbits-s-1} \\ \bit{x}{\bitind}=1}} 
     (C_{\mat{c}^{(\numlayers)}_c}=2^s\cdot x+M\cdot 2^{\numbits}) \right)
     \lor 
     (C_{\mat{c}^{(\numlayers)}_c}\geq 2^s\cdot x+M\cdot 2^{\numbits})\right)\right)
    \end{align*}

    Each $C_{\mat{q},\mat{c}^{(\numlayers)}_c} $ has size $O(p)$, girth $O(p)$, depth $O(1)$, and precision $O(p)$.
    When plugged into each $P_{\mat{c}_c^{(\numlayers)}=\vec{x}}$, there is a total contribution of size $O(pd2^{p^3d})$, depth $O(1)$, girth $O(p)$, and precision $O(p)$.

    Finally, to simulate $\mathbf{h}^{(\numlayers)}_{i}(\str{w}) = f^{(\numlayers)}\left(\mathbf{c}^{(\numlayers)}_{i}(\str{w}) + \mathbf{h}^{(\numlayers-1)}_i(\str{w})\right) $ we use the multi-ary case of \cref{lem:finite_function} once more, as we have simulated $\mathbf{c}^{(\numlayers)}_{i}$ and $\mathbf{h}^{(\numlayers-1)}_i$. 
    This again adds $O(p^2d^22^{p^2d})$ to the program size without changing the depth, girth, or precision.
    
    Overall we have increased the size additively by $O(p^2d2^{p^3d})$, increased the depth additively by $O(1)$, left the precision at $O(p)$, and left the girth at $O(p)$. 
    This results in a program with 
    \begin{itemize}
        \item Size $O(|\Sigma|Lp^2d2^{p^3d})$
        \item Depth $O(L)$
        \item Precision $O(p)$
        \item Girth $O(p)$
    \end{itemize}
\end{proof}

\subsection{Logic to Fixed-Precision Transformers}\label{app:logic_to_transformer}

\begin{theorem}\label{thm:crasppos_to_transformers}
    Let $\phi$ be a $\CRASPpos$ program with precision $p$, depth $L$, and size $d$. 
    There exists a \frtfr{} $T$ with precision $O(p)$, depth $O(L)$, and hidden dimension $O(d)$ which simulates $\phi$.
\end{theorem}
\begin{proof}
    All cases follow similarly as described in \citet{yang2025kneedeep, li2025characterizing} except for $ \phi=\sum _{t\in \mathcal{T}} \alpha_t t \sim k $, where $\alpha_i,k\in\N$ and $\mathord\sim\in\{<,\leq, =,\geq,>\}$. 
    We will just discuss the case of $\sim\in\{\geq\}$, from which all others are definable using boolean functions and adding a constant $+1$.

    Consider a subformula $\phi = \sum_{t\in \mathcal{T}} \alpha_t t\ge k$.
    We can separate out the constants terms to rewrite it in the form $\phi = C + \sum_{f\in \mathcal{F}} \alpha_f \countl[\phi_f]\ge k$, where $C\in\N$. 
    Assume that previous layers have computed $\mathbb{I}[w, j \models \phi_t]$ for $t \in \mathcal{T}$ at all positions $j \in [n]$. We want to
    construct a new layer that computes $\mathbb{I}[w,i \models \phi]$ at all positions $i \in [n]$.

    We pick a fixed-precision representation $\F$ which contains $\frac{1}{k}$ as well as $k$, and also contains a value $x$ such that $\round(\exp(x))=0$. 
    This is always possible.

    Using constructions from \citet{anonymous2025the}, it is possible to set $W_{\textnormal{Q}}$,$W_{\textnormal{K}}$, and $W_{\textnormal{V}}$ so that:

    \begin{align*}
        \round(\exp(s_{ij}))&=\begin{cases}
            1 & \text{$\str{w},j\models Q_{\bos}$}\\
            1 & \text{$\str{w},j\models \bigvee _{t\in\mathcal{T}}\phi_t$}\\
            0 & \text{otherwise}
        \end{cases}\\
        \mat{v}_j &= \begin{cases}
            -(k-1) + C & \text{$\str{w},j\models Q_{\bos}$}\\
            \sum_{t\in\mathcal{T}} \mathbb{I}[w, j \models \phi_t] &\text{otherwise}
        \end{cases}
    \end{align*}

    Then, after plugging into the definition of attention in \cref{eq:att} we get that:

    \begin{align*} 
        \mathbf{c}^{(\numlayers)}_{i}(\str{w}) &= \round\left(\frac{-(k-1)+\sum_{j=1}^{i}
        \sum_{t\in\mathcal{T}} \mathbb{I}[w, j \models \phi_t]}{|\{j\in[|\str{w}|]\colon \mathbb{I}\left[w, j \models\bigvee _{t\in\mathcal{T}}\phi_t \right]\}|+1}\right)
    \end{align*}

    We observe that $\mathbf{c}^{(\numlayers)}_{i}\geq 0 \iff \str{w},i\models \sum_{t\in \mathcal{T}} \alpha_t t\ge k$. 
    Thus, $\mathbf{c}^{(\numlayers)}_{i}$ rounds to $-2^{-\fracbits}$ or below if $\phi$ is false, and rounds to $0$ or above if $\phi$ is true. 
    We can then use the FFNN $f$ to map these two cases to $0$ or $1$, respectively. 
    This adds one layer, some constant number of hidden dimensions, and uses the same precision $p$.

\end{proof}

\section{Depth 2 Is Sufficient}\label{app:crasp_two}

\begin{proposition}
    There is a satisfiability-preserving reduction from \CRASP{} to \CRASP{} with depth at most 2.    
\end{proposition}

The proof of this proposition is similar to the so-called \emph{Scott's Normal Form} for first-order logic with two variables \cite{GKV97}. Given $\varphi \in \CRASP$, we construct $\varphi'$ with $\depth(\varphi') = 1$ and a set $S := S_{\varphi}$ of ``basic'' formulas such that $\varphi$ is satisfiable iff $\varphi' \wedge \bigwedge_{\psi \in S} \psi$ is satisfiable. The proof goes by induction on the depth of $\varphi$. If $\depth(\varphi) \leq 1$, then $S = \{\varphi\}$. If $\depth(\varphi) > 2$, then for each $\#\psi$ with $\depth(\psi) = 1$ occuring in $\varphi$, we have by induction constructed $\psi'$ with $\depth(\psi') = 1$ and the set $S_{\psi}$ of basic formulas. Starting with $S_{\varphi} = \emptyset$, we introduce a new proposition $P_{\psi}$ and do the following: (i) replace $\psi$ in $\varphi$ by $P_{\psi}$, and (ii) add to $S_{\varphi}$ the formula $\#((P_{\psi} \wedge \neg \psi') \vee (\neg P_{\psi} \wedge \psi'))=0$. This results in a formula $\varphi'$ with $\depth(\varphi') = 1$ and $S_{\varphi}$ such that $\varphi$ is satisfiable iff $\varphi' \wedge \bigwedge_{\psi \in S_{\varphi}} \psi$ is satisfiable.

\end{document}